\documentclass[twoside]{jmlr} 

\usepackage{times}

\usepackage{amsmath}
\usepackage{amssymb}

\newcommand{\F}{\mathcal{F}}

\newcommand{\real}{\mathbb{R}}

\newcommand{\OO}{\mathcal{O}}

\newcommand{\trace}[1]{\mbox{tr}\left(#1\right)}

\newcommand{\PP}[1]{\mathbb{P}\left[#1\right]}

\newcommand{\EE}[1]{\mathbb{E}\left[#1\right]}
\newcommand{\EEb}[1]{\mathbb{E}\bigl[#1\bigr]}
\newcommand{\EEtb}[1]{\mathbb{E}_t\bigl[#1\bigr]}

\newcommand{\EEs}[2]{\mathbb{E}_{#2}\left[#1\right]}
\newcommand{\EEu}[1]{\mathbb{E}_{\u}\left[#1\right]}

\newcommand{\EEo}[1]{\mathbb{E}_{0}\left[#1\right]}

\newcommand{\EEt}[1]{\mathbb{E}_t\left[#1\right]}

\newcommand{\EEcc}[2]{\mathbb{E}\left[\left.#1\right|#2\right]}

\newcommand{\ra}{\rightarrow}

\newcommand{\iprod}[2]{\left\langle#1,#2\right\rangle}
\newcommand{\biprod}[2]{\bigl\langle#1,#2\bigr\rangle}

\newcommand{\norm}[1]{\left\|#1\right\|}

\newcommand{\ev}[1]{\left\{#1\right\}}
\newcommand{\pa}[1]{\left(#1\right)}
\newcommand{\bpa}[1]{\bigl(#1\bigr)}

\newcommand{\wh}{\widehat}
\newcommand{\wt}{\widetilde}

\newcommand{\hx}{\wh{\boldsymbol{x}}}

\newcommand{\transpose}{^\mathsf{\scriptscriptstyle T}}

\definecolor{PalePurp}{rgb}{0.66,0.57,0.66}

\newcommand{\regret}{\mathrm{regret}}

\newcommand{\qed}{\hfill\BlackBox\\[2mm]}

\newcommand{\R}{\boldsymbol{R}}

\newcommand{\hL}{\wh{\L}}

\newcommand{\I}{\boldsymbol{I}}
\newcommand{\B}{\boldsymbol{B}}
\newcommand{\A}{\boldsymbol{A}}
\newcommand{\W}{\boldsymbol{W}}
\renewcommand{\P}{\boldsymbol{P}}
\newcommand{\Wh}{\widetilde{\boldsymbol{W}}}
\renewcommand{\H}{\boldsymbol{H}}
\newcommand{\V}{\boldsymbol{V}}
\newcommand{\U}{\boldsymbol{U}}
\newcommand{\hU}{\widetilde{\boldsymbol{U}}}

\renewcommand{\L}{\boldsymbol{L}}
\newcommand{\tL}{\widetilde{\boldsymbol{L}}}
\newcommand{\x}{\boldsymbol{x}}

\newcommand{\y}{\boldsymbol{y}}
\newcommand{\s}{\boldsymbol{s}}
\newcommand{\e}{\boldsymbol{e}}
\renewcommand{\u}{\boldsymbol{u}}

\renewcommand{\v}{\boldsymbol{v}}

\newcommand{\w}{\boldsymbol{w}}

\newcommand{\bmu}{\boldsymbol{\mu}}
\newcommand{\bnu}{\boldsymbol{\nu}}

\newcommand{\dif}{\,\mathrm{d}}
\DeclareMathOperator{\tr}{tr}
\DeclareMathOperator*{\argmin}{argmin}

\title[Bandit PCA]{Bandit Principal Component Analysis}

\author{\Name[{Wojciech~Kot\l owski}]{Wojciech Kot\l owski} \Email {wkotlowski@cs.put.poznan.pl} \\
        \addr Pozna\'n University of Technology, Pozna\'n, Poland\\
        \AND
 \Name[{Gergely~Neu}]{Gergely Neu} \Email{gergely.neu@gmail.com}\\
        \addr Universitat Pompeu Fabra, Barcelona, Spain
       }

\begin{document}

\maketitle

\begin{abstract}%
We consider a partial-feedback variant of the well-studied online PCA problem where a learner 
attempts to predict a sequence of $d$-dimensional vectors in terms of a quadratic loss, 
while only having limited feedback about the environment's choices. We focus on a 
natural notion of bandit feedback where the learner only 
observes the loss associated with its 
own prediction. Based on the classical observation that this 
decision-making problem can be lifted to the space of 
density matrices, we propose an algorithm 
that is shown to achieve a regret of $\wt{\OO}(d^{3/2}\sqrt{T})$ after $T$ rounds in the worst 
case. We also prove data-dependent bounds that improve on the basic result when the loss 
matrices of the environment have bounded rank or the loss of the best action is bounded. 
One version of our algorithm runs in $O(d)$ time per trial which massively improves over every 
previously known online PCA method.
We complement these results by a lower bound of 
$\Omega(d\sqrt{T})$. 
\end{abstract}


\section{Introduction}

Consider the problem of \emph{phase retrieval} where one is interested in reconstructing 
a unit-norm vector $\x \in \real^d$ up to a sign based on a 
number of noisy measurements of the form $|\w_t\transpose\x|^2$. Such problems arise 
abundantly in numerous areas of science and engineering such as in 
X-ray cristallography, astronomy, and diffractive imaging \citep{Mil90}. In the classical setting 
of phase retrieval, the measurement vectors 
$\w_t$ are typically drawn i.i.d.~from a distribution chosen before any measurements are taken 
\citep{Fie82,CSV13,S+15}. In the present paper, we study a sequential decision-making 
framework generalizing this classical problem to situations where the measurements can be chosen 
adaptively and the sequence of hidden vectors can be chosen by an adversary.

Our formulation can be most accurately described as a partial-information variant of the  
well-studied problem of online principal component analysis (online PCA) 
\citep{WarmuthKuzmin2006,WarmuthKuzmin2008,Nie_etal2016JMLR}. In the basic version of the online PCA 
problem, the learner receives a sequence of input vectors $\x_1,\x_2,\ldots,\x_T$, 
and is tasked 
with projecting these vectors one by one to a sequence of one-dimensional hyperplanes represented 
by the rank-one projection matrices 
$\P_t = \w_t\w_t\transpose$ (with $\|\w_t\| = 1$), in order to maximize the total squared norm of the projected
inputs, $\sum_t \|\P_t \x_t \|^2$.
Crucially, the learner selects each projection before observing 
the input vector, but nevertheless the input vector is fully revealed  to the learner at the end of 
each round. In our problem setup, we remove this last assumption and assume that the learner 
\emph{only observes the projection ``gain'' $\|\P_t \x_t\|^2$, but not the input vector $\x_t$}. By analogy to the 
multi-armed bandit problem, we will refer to this setting as \emph{bandit PCA}.

As already noted by \citet{WarmuthKuzmin2006}, the seemingly quadratic 
objective 
is in fact a linear function of the projection, $\|\P_t \x_t\|^2 = \tr(\P_t \x_t \x_t\transpose)$. 
Therefore, the bandit PCA problem can be reduced to a \emph{linear bandit problem}, in which the learner plays 
with a rank-one projection matrix $\P_t$, 
the environment chooses a symmetric \emph{loss matrix}
$\L_t = - \x_t \x_t\transpose$, and the learner suffers and observes loss $\tr(\P_t \L_t)$.
Using a generic algorithm for linear bandits,
the continuous version of the Exponential Weights algorithm 
\citep{DHK08,bubeck2014,vanDerHoeven_etal_2018},
one can achieve a regret bound of order $\OO(p \sqrt{T \ln T})$, where $p$ is the dimension of the action and loss
spaces.
Unfortunately, the algorithm is computationally inefficient as it
needs to maintain and update a distribution over the continuous set of rank-one projection matrix. 
Furthermore, observe that $p=\OO(d^2)$ in our setup, so the regret 
bound is in fact quadratic in the dimension of the problem.

In this paper, we address both of the above shortcomings and propose an efficient 
algorithm for a generalization of the bandit PCA problem in which the adversary is allowed 
to play symmetric loss matrices of arbitrary rank. Our algorithm
achieves a regret bound of $\OO(d \sqrt{r T \ln T})$, where $r$ is the 
average squared Frobenious norm of the loss matrices played by the environment
(which is upper bounded by their maximal rank of these matrices).
Our regret bound improves the 
one mentioned above by at least a factor of $\sqrt{d}$, and can achieve a factor of $d$ improvement 
when the Frobenius norm of the losses is bounded by a constant (e.g., in the 
original PCA case when all $\L_t$ have rank 
one). We complement our results with a lower bound of $\Omega(d\sqrt{T})$, leaving a factor of 
$\sqrt{d}$ gap between the two bounds in general. An interesting consequence of our lower bound is 
that it formally confirms the intuition that the bandit PCA problem is \emph{strictly harder} than 
the $d$-armed bandit problem where the minimax regret is of order $\Theta(\sqrt{dT})$ 
\citep{ACBFS02,AB10}. These results are to be contrasted with the fact that the full-information 
online PCA problem is \emph{exactly as hard} as the problem of prediction with expert advice, the 
minimax regret being of $\Theta(\sqrt{T\log d})$ in both cases \citep{Nie_etal2016JMLR}. 

On the front of computational complexity, one version of our algorithm achieves a surprisingly massive improvement 
over every previously known online PCA algorithm. Specifically, 
our algorithm only 
requires $\tilde{\OO}(d)$ computation per iteration, amounting to \emph{sublinear} runtime in the 
dimension of the action space $p = \OO(d^2)$. This striking runtime complexity should be contrasted 
with the full-information setup, in which the regret-optimal algorithms can only guarantee 
$\OO(d^{\omega})$\footnote{Time needed for matrix multiplication,
which is also the complexity of eigendecomposition with distinct eigenvalues \citep{AZL17}.}
per-round complexity
for full-rank loss matrices \citep{WarmuthKuzmin2008,AZL17}. In 
fact, full information algorithms all face the computational bottleneck of having to read out the 
entries of $\L_t$, which already takes $\OO(d^2)$ time. In contrast, our partial-information 
setup stipulates that nature computes and communicates the realized loss for the learner at no 
computational cost. 
We note that our algorithms can be 
readily adjusted to cope with noisy observations, which enables the use of fast randomized linear 
algebra methods for computing the losses.

Our algorithm is based on the generic algorithmic template of online mirror descent (OMD) 
\citep{NY83,BeckTeboulle2004,Hazan_OCO,JGS17}.
Similarly to the methods for the full-information variant of online PCA 
\citep{Nie_etal2016JMLR}, the algorithm maintains in each trial $t=1,\ldots,T$ 
a \emph{density matrix} $\W_t$ 
as a parameter, which is a positive definite matrix with unit trace, and
represents a mixture over rank-one projections. In each trial $t$, a projection 
$\w_t\w_t\transpose$ is sampled in such a way that its expectation matches the density matrix,
$\EE{\w_t \w_t\transpose} = \W_t$. Based on the observed loss, the algorithm constructs
an unbiased estimate $\tL_t$ of the unknown loss matrix $\L_t$, which is then
used to update the density matrix to $\W_{t+1}$.

The recipe described above is standard in the bandit literature, with a few degrees of 
freedom in choosing the regularization function for OMD, the scheme for sampling $\w_t$, and the 
structure of the loss estimator $\hL_t$.  While it may appear tempting to draw 
inspiration from existing full-information online PCA algorithms to make these design choices, it 
turns out that none of the previously used techniques are applicable in our setting.
In particular, the previously employed 
methods of sampling from a density matrix \citep{WarmuthKuzmin2006,WarmuthKuzmin2008}
by selecting eigendirections with probabilities equal to the eigenvalues turns out
to be insufficient, as it is only able to sense the diagonal elements of the loss matrix (when 
expressed in the eigensystem of the learner's density matrix), making it 
impossible to construct an unbiased loss estimator. Therefore, our first key algorithmic 
tool is designing a more sophisticated sampling scheme for $\w_t$ and a corresponding loss 
estimator.
Furthermore, we observe that the standard choice of the quantum negative entropy as the OMD 
regularizer \citep{meg} fails to provide the desired regret bound, no matter what unbiased loss 
estimator is used.
Instead, our algorithm is crucially based on using the negative log-determinant $-\log \det(\W)$ as 
the regularization function.

\subsection{Related work}
Our work is a direct extension of the line of research on online PCA
initiated by \citet{WarmuthKuzmin2006} and further studied by 
\citet{WarmuthKuzmin2008,Nie_etal2016JMLR}. Online PCA is an instance of the more general class of 
online matrix prediction problems, where the goal of the learner is to minimize its regret against 
the best matrix prediction chosen in hindsight \citep{meg,GHM15,AZL17}. \citet{BGKL15} studied 
another flavor of the online PCA problem where the goal of the learner is to encode a sequence of 
high-dimensional input vectors in a smaller representation.

Besides the above-mentioned works on online matrix prediction with full information, there is 
little existing work on the problem under partial information. One notable exception is the work of 
\citet{GRESS16} that considers a problem of reconstructing the top principal components of a 
sequence of vectors $\x_t$ while observing $r\ge 2$ arbitrarily chosen entries of the 
$d$-dimensional inputs. 
\citeauthor{GRESS16} propose an algorithm based on the Matrix Exponentiated Gradient method and 
analyze its sample complexity through regret analysis and an online-to-batch conversion. Their 
analysis is greatly facilitated by the observation model that effectively allows a decoupling of 
exploration and exploitation, since the loss of the algorithm is only very loosely related to the 
chosen observations. In contrast, our setting presents the learner with a much more challenging 
dilemma since the observations are strictly tied to the incurred losses, and our feedback only 
consists of a single real number instead of $r\ge 2$. This latter difference, while seemingly 
minor, can often result in a large gap between the attainable regret guarantees 
\citep{ADX10,HPGS16}.

Another closely related problem setting dubbed ``rank-1 bandits'' was considered by 
\citet{KKSVW17}, \citet{KSRWAM17}, and \citet{JWWN19}. In these problems, the learner is tasked 
with choosing two $d$-dimensional decision vectors $\x_t$ and $\y_t$, and obtains a reward that is 
a 
bilinear function of the chosen vectors: $\x_t \transpose \R_t \y_t$ for some matrix $\R_t$. The 
setup most closely related to ours is the one considered by \citet{JWWN19}, who assume arbitrary 
 action sets for the learner and prove regret bounds of order $d^{3/2}\sqrt{r T}$, 
where $r$ is the rank of the reward matrix. Notably, these results assume 
that $\R_t$ is generated i.i.d.~from some unknown distribution. These results are to be contrasted 
with our bounds of order $d\sqrt{rT}$ that are proven for adversarially chosen loss matrices. Note 
however that the two results are not directly comparable due to the mismatch between the considered 
decision sets: our decision set is in some sense smaller but more complex due to the semidefinite 
constraint, whereas theirs is larger but has simpler constraints.

Our analysis heavily draws on the literature on non-stochastic multi-armed 
bandits \citet{ACBFS02,AB10,BubeckCesaBianchi2012}, and makes particular use of the regularization 
function commonly known as the \emph{log-barrier} (or the \emph{Burg entropy}) that has been 
recently applied with great success to solve a number of challenging bandit problems 
\citep{FLLST16,ALNS17,BCL18,CYL18,LWZ18}. Indeed, our log-determinant regularizer is a direct 
generalization of the log-barrier function to the case of matrix-valued predictions, where the 
induced Bregman divergence is often called Stein's loss. This loss function is 
commonly used in covariance matrix estimation in statistics \citep{JamesStein1961} and online 
metric learning \citep{DKJSD07,JKDG09,KB10}. 

Finally, let us comment on the close relationship between our setting and that of phase retrieval, 
already alluded to at the very beginning of this paper. Indeed, the connection is readily apparent 
by noticing that the quadratic gain $|\w_t\transpose\x|^2$ is equivalent to the projection gain
$\norm{\P_t\transpose\x}^2$, amounting to a bandit PCA problem instance with loss 
matrix $-\x\x\transpose + \xi_t \I$, where the last term serves to model observation noise. A 
typical goal of a phase retrieval algorithm is to output a vector $\hx$ that minimizes the 
distance $\min\norm{\x\pm\hx}$ to the hidden signal $\x$. It is easy to show that our regret bounds 
of minimax order $\sqrt{T}$ translate to upper bounds of order $T^{-1/4}$ through an 
simple online-to-batch conversion, matching early results on phase retrieval by \citet{EM14}. 
However, more recent results show that the true minimax rates are actually of 
$\Theta\bpa{T^{-1/2}}$ \citep{LM15,CLM16}. This highlights that in some sense the online 
version of this problem is much harder in that minimax rates for the regret do not seem to directly 
translate to minimax rates on the excess risk under i.i.d.~assumptions. 

\paragraph{Notation.} $\mathcal{S}$ is the set of $d\times d$ symmetric positive semidefinite 
(SPSD) matrices and $\mathcal{W} \subset \mathcal{S}$ is the set of density matrices $\W$ satisfying 
$\tr(\W) = 1$. We will use the notation $\iprod{\A}{\B} = \trace{\A\transpose\B}$ for any two $d\times 
d$ matrices  $\A$ and $\B$, and define the Frobenius norm of any matrix $\A$ as $\norm{\A}_F = 
\sqrt{\iprod{\A}{\A}}$. We will consider randomized iterative algorithms that interact with a 
possibly 
random environment, giving rise to a filtration $\pa{\F_t}_{t\ge 1}$. We will often use the 
shorthand $\EEt{\cdot} = \EEcc{\cdot}{\F_t}$ to denote expectations conditional on the 
interaction history.

\section{Preliminaries}

We consider a sequential decision-making problem where a \emph{learner} interacts with its 
\emph{environment} by repeating the following steps in a sequence of rounds $t=1,2,\dots,T$:
\begin{enumerate}
  \item learner picks a vector $\w_t\in\real^d$ with unit norm, possibly in a randomized
    way,
  \item environment picks a loss matrix $\L_t$ with spectral norm bounded by $1$,
 \item learner incurs \emph{and observes} loss $\iprod{\w_t\w_t\transpose}{\L_t} = \tr\pa{\w_t\w_t\transpose\L_t}$.
\end{enumerate}
Note that the crucial difference from the traditional setup of online PCA is that the learner does 
not get to observe the full loss matrix $\L_t$. We will make the most minimal assumptions about the 
environment: the loss $\L_t$ in round $t$ is allowed to depend on the entire interaction history 
except the last decision $\w_t$ of the learner. In other words, we will consider algorithms that 
work against \emph{non-oblivious} or \emph{adaptive adversaries}.

The performance of the learner is measured in terms of the \emph{total expected regret} (or, 
simply, the \emph{regret}), which is the difference between the cumulative loss
of the algorithm and that of a fixed action optimal in expectation:
\[
 \mathrm{regret}_T = \max_{\u: \norm{\u} = 1} \sum_{t=1}^T \EE{\iprod{\w_t\w_t\transpose - 
\u\u\transpose}{\L_t}},
\]
where the expectation is with respect to the internal randomization of the learner\footnote{This 
definition of regret is sometimes called \emph{pseudo-regret} \citep{BubeckCesaBianchi2012}.}.

\section{Algorithms and main results}
\label{sec:algorithm}

This section presents our general algorithmic template, based on the generic algorithmic framework 
of online mirror descent \citep{NY83,BeckTeboulle2004,Hazan_OCO,JGS17}. Such algorithms 
are 
crucially based on a choice of a differentiable convex regularization function 
$R:\mathcal{S} \ra \real$ and the associated Bregman divergence $D_R : \mathcal{S} \times 
\mathcal{S} \ra \real_+$ induced by $R$:
\[
  D(\W \| \W') = R(\W) - R(\W') - \iprod{\nabla R(\W')}{(\W - \W')}.
\]
Our version of online mirror descent proceeds by choosing the initial density matrix as
$\W_1  = \frac{1}{d}\I \in \mathcal{W}$, and then iteratively computing the sequence of density 
matrices
\[
 \W_{t+1} = \argmin_{\W \in \mathcal{W}} \left\{\eta \bigl\langle\W,\tL_t\bigr\rangle + D_R(\W \| 
\W_t) \right\}.
\]
Here, $\tL_t \in \mathcal{S}$ is an estimate of the loss matrix $\L_t$ chosen by the environment in 
round $t$. Having computed $\W_t$, the algorithm randomly draws the unit-norm vector $\w_t$ 
satisfying $\EEt{\w_t\w_t} = \pa{1-\gamma}\W_t + \frac{\gamma}{d} \I$, where the latter term is 
added to prevent the eigenvalues of the covariance matrix from approaching $0$. This effect 
is modulated by the parameter $\gamma \in [0,1]$ that we will call the \emph{exploration 
rate}.
The main challenges posed by our particular 
setting are:
\begin{itemize}
 \item finding a way to sample a unit-length vector $\w_t$ satisfying $\EEt{\w_t\w_t\transpose} = \pa{1-\gamma}\W_t + \frac{\gamma}{d} \I$,
 \item constructing a suitable (hopefully unbiased) loss estimator $\tL_t$ based on the observed 
loss $\ell_t = \iprod{\w_t\w_t\transpose}{\L_t}$ and the vector $\w_t$,
 \item finding a regularization function $R$ that is well-adapted to the previous design choices.
\end{itemize}
It turns out that addressing each of these challenges will require some unusual techniques. The 
most crucial element is the choice of regularization function 
that we choose as the negative log-determinant $R(\W) = - \log \det(\W)$, with its derivative given as $-\W^{-1}$ and the associated
Bregman divergence being
\[
  D_R(\W \| \U) = \tr(\U^{-1} \W) - \log\det(\U^{-1} \W) - d,
\]
which is sometimes called \emph{Stein's loss} in the literature, and coincides with the relative 
entropy between the distributions $\mathcal{N}(0,\W)$ and $\mathcal{N}(0,\U)$. In contrast
to the multi-armed bandit setting, here the choice of the right regularizer turns out to be
much more subtle, as the standard choices of the quantum negative entropy \citep{meg} or matrix Tsallis entropy
\citep{ALO2015} fail to provide the desired regret bound
(a discussion on these issues is included in Appendix \ref{sec:MH_Tsallis}). 

For sampling the vector $\w_t$, 
a peculiar challenge in our problem is having to design a process that will allow constructing an 
unbiased estimator of the loss matrix $\L_t$. To this end, we propose two
different sampling strategies along with their 
corresponding loss-estimation schemes based on the eigendecomposition of the density matrices.
The two strategies will 
be later shown to achieve two distinct flavors of data-dependent regret bounds. We present the 
details of these sampling schemes below in a simplified notation: given the eigenvalue
decomposition $\W = \sum_i \lambda_i \u_i \u_i\transpose$ of density matrix $\W$, the procedures sample $\w$ such that
$\EE{\w\w\transpose}=\W$, and construct the loss estimate $\tL$ for which $\mathbb{E}[\tL]= 
\L$. Recall that we use $\W = (1-\gamma) \W_t + \frac{\gamma}{d} \I$ in the algorithm.

\begin{figure}[ht]
\begin{algorithm2e}[H]
\SetAlgoNoLine
\SetAlCapHSkip{0pt}
\label{alg:banditpca}
\caption{Online Mirror Descent for Bandit PCA}
\DontPrintSemicolon
\SetAlgoNoEnd
\SetKwInOut{Parameter}{Parameters}
\SetKwInOut{Initialization}{Initialization}
\Parameter{learning rate $\eta > 0$, exploration rate $\gamma \in [0,1]$}
\Initialization{$\W_1 \leftarrow \frac{\I}{d}$}
 \For{$t=1,\ldots,T$}{
   eigendecompose $\W_t = \sum_{i=1}^d \mu_i \u_i \u_i\transpose$\;
   $\boldsymbol{\lambda} \leftarrow (1-\gamma) \boldsymbol{\mu} 
    + \gamma \big(\frac{1}{d},\ldots,\frac{1}{d}\big)$ \;
   \SetKwFunction{Sample}{\bfseries sample}
   $\tL_t \leftarrow$ \Sample\!$\big(\boldsymbol{\lambda}, \; \{\u_i\}_{i=1}^d\big)$\;
   $\W_{t+1} \leftarrow \big(\W_t^{-1} + \eta \tL_t + \beta \I\big)^{-1}$
   with $\beta$ such that $\tr(\W_{t+1})=1$ \;
 }
\end{algorithm2e}

\begin{minipage}{.5\textwidth}
\begin{algorithm2e}[H]
\SetAlgoNoLine
\SetAlCapHSkip{0pt}
\label{alg:densesampling}
\caption{Dense sampling}
\DontPrintSemicolon
\SetAlgoNoEnd
\SetKwFunction{Sample}{\bfseries sample}
\SetKwProg{Fn}{def~}{\string:}{}
\Fn{\Sample\!$\big(\boldsymbol{\lambda}, \; \{\u_i\}_{i=1}^d\big)$}{
   $B \sim \mathrm{Bernoulli}\big(\frac{1}{2}\big)$\;
   \eIf{$B=1$}{draw $I \sim \boldsymbol{\lambda}$ and set $\w_t \leftarrow \u_I$}
   {draw $\s \in \{-1,+1\}^d$ i.i.d. uniformly \;
   $\w_t \leftarrow \sum_i s_i \sqrt{\lambda_i} \u_i$}
   play $\w_t$ and observe $\ell_t =\iprod{\w_t\w_t\transpose}{\L_t}$ \;
   \eIf{$B=1$}{$\tL_t \leftarrow 2\ell_t \W_t^{-1/2} \w_t \w_t\transpose \W_t^{-1/2}$}
   {$\tL_t \leftarrow \ell_t \big(\W_t^{-1} \w_t \w_t\transpose \W_t^{-1} - \W_t^{-1}\big)$}
   \Return{$\tL_t$}
 }
\end{algorithm2e}
\end{minipage}
\begin{minipage}{.5\textwidth}
\begin{algorithm2e}[H]
\SetAlgoNoLine
\SetAlCapHSkip{0pt}
\label{alg:sparsesampling}
\caption{Sparse sampling}
\DontPrintSemicolon
\SetAlgoNoEnd
\SetKwFunction{Sample}{\bfseries sample}
\SetKwProg{Fn}{def~}{\string:}{}
\Fn{\Sample\!$\big(\boldsymbol{\lambda}, \; \{\u_i\}_{i=1}^d\big)$}{
  draw $I,J\sim \boldsymbol{\lambda}$\;
   \eIf{$I=J$}{$\w_t \leftarrow \u_I$}
   {draw $s\in\ev{-1,1}$ uniformly \;
  $\w_t\leftarrow \frac{1}{\sqrt{2}}\pa{\u_I + s\u_J}$}
   play $\w_t$ and observe $\ell_t =\iprod{\w_t\w_t\transpose}{\L_t}$ \;
   \eIf{$I=J$}{$\tL_t \leftarrow \pa{\ell_t / \lambda_I^2} \u_I\u_I\transpose$}
   {
   $\tL_t \leftarrow s\ell/ \pa{ 2\lambda_I\lambda_J} \pa{\u_I\u_J\transpose + \u_J\u_I\transpose}$}
   \Return{$\tL_t$}
 }
\end{algorithm2e}
\end{minipage}
\end{figure}

\subsection{Dense sampling}
Our first sampling scheme is composed of two separate sampling procedures,
designed to sense and estimate the on- and off-diagonal 
entries of the loss matrix $\L$ (when expressed in the eigensystem of $\W$), respectively. 
Precisely, the procedure will first draw a Bernoulli random variable $B$ with 
$P(B=1) = \frac{1}{2}$, and sample $\w$ depending on the outcome as follows:
\begin{itemize}
  \item If $B = 1$, sample $\w$ as one of the eigenvectors $\u_I$ such that $\PP{I=i} = \lambda_i$. 
This clearly gives $\EE{\w\w\transpose} = \sum_{i=1}^d 
\lambda_i \u_i \u_i\transpose= \W$.
  \item If $B=0$, draw i.i.d.~uniform random signs $\s = (s_1,\ldots,s_d) \in \{-1,+1\}^d$ and
  sample $\w$ as
    \[
      \w = \sum_{i=1}^d s_i \sqrt{\lambda_i} \u_i.
    \]
    Note that $\|\w\| = 1$ and we have
    \[
    \EE{\w\w\transpose}
    = \mathbb{E}_{\s}\Big[\sum_{ij} s_i s_j \sqrt{\lambda_i \lambda_j} \u_i \u_j\transpose\Big]
    = \sum_{ij} \underbrace{\EEs{s_i s_j}{\s}}_{\delta_{ij}} \sqrt{\lambda_i \lambda_j} \u_i \u_j\transpose 
    =  \sum_i \lambda_i \u_i \u_i\transpose = \W.
   \]
\end{itemize}
The first method is the standard sampling procedure in the full-information
version of online PCA. In the bandit case, this method turns out to be
insufficient,
as it only let us observe $\iprod{\u_i \u_i\transpose}{\L}  = 
\u_i\transpose \L \u_i$, that is, the on-diagonal elements of the loss matrix $\L$ expressed in 
the eigensystem of $\W$. 
On the other hand, the second method does sense the off-diagonal elements $\u_i\transpose \L \u_j$, 
but misses the on-diagonal ones. Thus, a combination of the two methods is sufficient for 
recovering the entire matrix.
We will refer to this sampling method as \emph{dense} since it observes a dense linear 
combination of the off-diagonal elements of the matrix $\L$. Having observed
$\ell = \iprod{\w\w\transpose}{\L}$, we construct our 
estimates in the two cases corresponding to the outcome of the random coin flip $B$ as follows:
\[
\tL = \left\{
  \begin{array}{ll}
    2 \ell \W^{-1/2} \w \w\transpose \W^{-1/2} &\qquad \text{if~~} B=1, \\[1mm]
    \ell \pa{\W^{-1} \w \w\transpose \W^{-1} - \W^{-1}} &\qquad \text{if~~} B=0.
  \end{array}
\right.
\]
The following lemma (proved in Appendix~\ref{app:dense_bias}) shows that the above-defined 
estimate is unbiased.
\begin{lemma} \label{lem:dense_bias}
The estimate $\tL_t$ defined through the dense sampling method satisfies 
$\mathbb{E}_t\tL_t = \L_t$.
\end{lemma}

\subsection{Sparse sampling} Our second method is based on sampling two eigenvectors of $\W$ with 
indices $I$ and $J$ independently from the same distribution satisfying $\PP{J = i} = \PP{I = i} = 
\lambda_i$. Then, when $I=J$, it selects $\w = \u_I$, whereas for $I \neq J$, 
it draws a uniform random sign $s \in \{-1,1\}$ and
sets $\w = \frac{1}{\sqrt{2}}(\u_I + s \u_J)$.
We refer to this procedure as \emph{sparse} since the observed loss 
is a sparse linear combination 
of diagonal and off-diagonal elements. 
We first verify that this method indeed satisfies $\EE{\w \w\transpose} = \W$:

\begin{align*}
  \EE{\w \w\transpose} &= \underbrace{\sum_i \lambda_i^2 \u_i \u_i\transpose}_{\text{when~} I=J}
  + \underbrace{\sum_{i \neq j} \lambda_i \lambda_j \frac{1}{2}
  \EEs{(\u_i + s \u_j)(\u_i + s \u_j)\transpose}{s}}_{\text{when~} I \neq J} \\
  &= \sum_i \lambda_i^2 \u_i \u_i\transpose
  + \frac{1}{2} \sum_{i \neq j} \lambda_i \lambda_j 
  \Big(\u_i \u_i\transpose 
  +  \u_j \u_j\transpose \Big)
  = \sum_{ij} \lambda_i \lambda_j \u_i \u_i\transpose = \sum_i \lambda_i \u_i \u_i\transpose 
= \W,
\end{align*}
where in the second equality we used the fact that $s^2=1$ and $\EEs{s}{s} = 0$.
The loss estimate is constructed as follows:
\[
\tL ~=~ \left\{
  \begin{array}{ll}
  \frac{\ell}{\lambda_I^2} \u_I \u_I\transpose & \quad \text{when~} I = J,   \\
  \frac{s \ell}{2 \lambda_I \lambda_J} (\u_I \u_J\transpose + \u_J \u_I\transpose) & \quad
  \text{when~} I \neq J.
  \end{array}
  \right.
\]
As the following lemma shows, this estimate is also unbiased. The proof is found in 
Appendix~\ref{app:sparse_bias}.
\begin{lemma} \label{lem:sparse_bias}
The estimate $\tL_t$ defined through the sparse sampling method satisfies 
$\mathbb{E}_t\tL_t = \L_t$.
\end{lemma}

\subsection{Upper bounds on the regret}
We can now state our main results regarding the performance of our algorithm with the two sampling 
schemes. Our first result is a data-dependent regret bound for the dense 
sampling method.
\begin{theorem}\label{thm:upper_dense}
Let $\eta \leq \frac{1}{2d}$ and $\gamma = 0$. The 
regret of Algorithm~\ref{alg:banditpca} with dense sampling satisfies
\[
\regret_T \le \frac{d \log T}{\eta} + \eta (d^2+1) \sum_{t=1}^T \EE{\ell_t^2} + 2.
\]
\end{theorem}
We can immediately derive the following worst-case guarantee from the above result: 
\begin{corollary}\label{cor:upper_dense_worstcase}
Let $\eta = \min\ev{\sqrt{\frac{\log T}{dT}},\frac{1}{2d}}$ and $\gamma=0$. Then, the 
regret of Algorithm~\ref{alg:banditpca} with dense sampling satisfies
 \[
\mathrm{regret}_T = \OO\left(d^{3/2} \sqrt{T \log T}\right)
\]
\end{corollary}
\begin{proof}
The claim is trivial when $\sqrt{\pa{\log T}/\pa{dT}} \geq 1/2d$.
Otherwise we use Theorem \ref{thm:upper_dense} 
together with $\ell_t^2 \leq 1$ and plug in the choice of $\eta$. 
\end{proof}
It turns out that the above bound can be significantly improved if we make some assumptions about 
the losses. Specifically, when the losses are assumed 
to be non-negative and there is a known upper bound on the cumulative loss of the best action: 
$\overline{L}_T^* \ge \min_{\u:\norm{\u}=1} \sum_t \tr(\u\u\transpose \L_t)$, a properly tuned variant of our 
algorithm satisfies the following first-order regret bound (proof in Appendix
\ref{appendix:corollary_dense}):
\begin{corollary}\label{cor:upper_dense_lstar}
Assume that $L_t$ is positive semidefinite for all $t$ and $\overline{L}_T^*$ is defined as above. 
Then for $\eta = \min\ev{\sqrt{\frac{\log 
T}{d\overline{L}_T^*}},\frac{1}{4d^2}}$, the regret of Algorithm~\ref{alg:banditpca}
with dense sampling satisfies
 \[
\mathrm{regret}_T = \OO\left(d^{3/2} \sqrt{\overline{L}_T^* \log T} + d^3 \log T \right) 
\]
\end{corollary}
Let us now turn to the version of our algorithm that uses the sparse sampling scheme.
\begin{theorem}\label{thm:upper_sparse}
Let $\eta \leq \frac{1}{2d}$ and $\gamma = \eta d$. The 
regret of Algorithm~\ref{alg:banditpca} with sparse sampling satisfies
\[
\regret_T \le \frac{d \log T}{\eta} + 2\eta d + 2 + 8 \eta d \sum_{t=1}^T \EE{\|\L_t\|^2_F}
\]
\end{theorem}
\begin{corollary}\label{cor:upper_sparse}
Let $r \ge \frac{1}{T} \sum_{t=1}^T \EE{\|\L_t\|_F^2}$ be known to the algorithm. Then, 
for $\eta = \min\ev{\sqrt{\frac{\log T}{rT}},\frac{1}{2d}}$ and $\gamma = d \eta$, the 
regret of the algorithm with sparse sampling satisfies
 \[
\mathrm{regret}_T = \OO\left(d \sqrt{r T \log T}\right) 
\]
\end{corollary}
\begin{proof}
The claim is trivial when $\sqrt{\pa{\log T}/\pa{rT}} \geq 1/{2d}$.
Otherwise we use Theorem \eqref{thm:upper_sparse} and plug in the choice of $\eta$.
\end{proof}
Note that since the spectral norm of the losses is bounded by $1$, we have $\|\L_t\|^2_F \leq \mathrm{rank}(\L_t)$.
Thus, for the classical online PCA problem in which $\L_t = -\x_t\x_t\transpose$, the bound becomes
$\OO(d \sqrt{T \log T})$.

\subsection{Lower bound on the regret}
We also prove the following lower bound on the regret of any algorithm:
\begin{theorem}\label{thm:lower}
There exists a sequence of loss matrices such that the regret of any algorithm is lower bounded as
 \[
  \mathrm{regret}_T \ge \frac{1}{16} d\sqrt{T/\log T}.
 \]
\end{theorem}
The proof can be found in Appendix \ref{appendix:lower_bound}. Note that
there is gap of order $\sqrt{d}$
between the lower bound and the upper bounds achieved by our
algorithms.

\section{Analysis}

\label{sec:analysis}
This section presents the proofs of our main results. We decompose the proofs into two main parts: 
one considering the regret of online mirror descent with general loss estimators, and another part 
that is specific to the loss estimators we propose.

For the general mirror descent analysis, it will be useful to rewrite the update in 
the following form:
\begin{equation}
\begin{array}{rl}
  \text{(update step)} &\displaystyle \Wh_{t+1} = \argmin_{\W} 
  \left\{D_R(\W \| \W_t) + \eta \tr(\W \tL_t) \right\},   \\[2mm]
  \text{(projection step)} & \displaystyle \W_{t+1} 
  = \argmin_{\W \in \mathcal{W}} D_R(\W \| \Wh_{t+1}),
\end{array}
\label{eq:MD_update}
\end{equation}
where $\mathcal{W}$ is the set of density matrices. 
The unprojected solution $\Wh_{t+1}$ can be shown to satisfy the equality $\nabla R(\Wh_{t+1}) = \nabla R(\W_t) - \eta \tL_t$, which gives\footnote{While we do 
not show it explicitly here, it will be apparent from the proof of Lemma~\ref{lem:deltabound} that this update is well-defined since $\W_t^{-1} + \eta \tL_t$ 
is invertible under our choice of parameters.}
\begin{equation}\label{eq:unprojected}
  \Wh_{t+1} = \left(\W_t^{-1} + \eta \tL_t \right)^{-1}
  = \W_t^{1/2} \left(\I + \eta \W_t^{1/2} \tL_t \W_t^{1/2}\right)^{-1} \W_t^{1/2}.
\end{equation}
Our analysis will rely on the result below that follows from a direct application
of well-known regret bound of online mirror descent, and a standard trick to relate the regret on 
the true and estimated losses, originally due to \citet{ACBFS02}.
\begin{lemma} \label{lem:mdbound2}
For any $\eta > 0$ and $\gamma \in [0,1]$, the regret of Algorithm~\ref{alg:banditpca} satisfies
\[
\mathrm{regret}_T
\leq
\frac{d \log T}{\eta} + 2 \gamma T + 2 + (1-\gamma) \sum_{t=1}^T \EE{\biprod{\W_t 
- \Wh_{t+1}}{\tL_t}}.
\]
\end{lemma}
The proof is rather standard and is included in 
Appendix~\ref{app:mdbound}.
The main challenge is bounding the last term in the above equation. To ease further 
calculations, 
we rewrite this term with the help of the matrix $\B_t = \W_t^{1/2} \tL_t 
\W_t^{1/2}$. From the definition of $\Wh_{t+1}$, we have
\[
  \Wh_{t+1} = \W_t^{1/2} (\I + \eta \B_t)^{-1} \W_t^{1/2} = \W_t - \eta \W_t^{1/2} \B_t (\I + \eta 
\B_t)^{-1} \W_t^{1/2},
\]
where the second equality uses the easily-checked identity $(\I + \A)^{-1} = \I - \A (\I + 
\A)^{-1}$. Therefore, the term in question can be written as
\begin{align}
\iprod{\W_t - \Wh_{t+1}}{\tL_t} 
  &= \eta \tr\left( \W_t^{1/2} \B_t (\I + \eta \B_t)^{-1} \W_t^{1/2} \tL_t \right) 
  = \eta \tr\left( \B_t (\I + \eta \B_t)^{-1} \B_t\right)  \nonumber \\
  &= \sum_{i=1}^d \eta \frac{b_{t,i}^2}{1 + \eta b_{t,i}}, 
  \label{eq:Bdecomp}
\end{align}
where $\{b_{t,i}\}_{i=1}^d$ are the eigenvalues of $\B_t$. We now separately bound \eqref{eq:Bdecomp} for
the dense and the sparse sampling method.

\subsection{Analysis of the dense sampling method}

\begin{lemma}\label{lem:deltabound}
Suppose that $\eta \le \frac{1}{2d}$ and $\gamma = 0$. Then, the dense sampling method guarantees
 \[
  \biprod{\W_t - \Wh_{t+1}}{\tL_t} \le
   \begin{cases}
  \frac{8}{3} \eta \ell^2_t & \mbox{if $B = 1$,} \\
  2 \eta d^2 \ell_t^2  & \mbox{if $B=0$.}
   \end{cases}
 \]
 In particular, the expectation is bounded as
 \[
  \EEt{\biprod{\W_t - \Wh_{t+1}}{\tL_t}} \le \eta \pa{d^2+1} \ell_t^2.
 \]
\end{lemma}
\begin{proof}
Let $\W_t = \sum_{i=1}^d \lambda_i \u_i \u_i\transpose$ be the eigendecomposition of $\W_t$.
Note that due to the assumption that $\L_t$ has spectral norm bounded by $1$,
$|\ell_t| = |\tr(\L_t \w_t \w_t\transpose)| \leq \|\L_t\|_{\infty} \tr(\w_t \w_t\transpose)
\leq 1$.
We prove the bound separately for the two cases corresponding to the different values of $B$.

\paragraph{On-diagonal sampling ($B=1$).} When $B=1$, we have
\[
  \tL_t = 2 \ell_t \W_t^{-1/2} \u_i \u_i\transpose \W_t^{-1/2},
\]
for some $i \in \{1,\ldots,d\}$,
so that $\B_t = 2 \ell_t \u_i \u_i^\top$ is rank-one, with
single nonzero eigenvalue $b_{t,1} = 2 \ell_t$.
Using \eqref{eq:Bdecomp} gives
\[
  \biprod{\W_t - \Wh_t}{\tL_t} = \frac{4 \eta \ell_t^2}{1 + 2 \eta \ell_t},
\]
and the claimed result follows by noticing that our assumption on $\eta$ guarantees $|\eta \ell_t| \leq \frac{1}{2d} \leq \frac{1}{4}$.
\paragraph{Off-diagonal sampling ($B=0$).} We now have
\[
  \tL_t =  \ell_t (\W_t^{-1} \w_t \w_t\transpose \W_t^{-1} - \W_t^{-1}),
  \quad \text{where~~} \w_t = \sum_{i=1}^d s_i \sqrt{\lambda_i} \u_i.
\]
Denoting $\v = \W_t^{-1/2} \w_t = \sum_{i=1}^d s_i \u_i$, we get
\[
  \B_t = \W_t^{1/2} \tL_t \W_t^{1/2} = \ell_t (\v \v^\top - \I)
\]
Using orthonormality of $\{\u_i\}_{i=1}^T$ we have $\|\v\|^2 = \sum_i s_i^2 = d$,
which means that $\B_t$ has a single eigenvalue $\ell_t(d-1)$,
with the remaining $d-1$ eigenvalues all equal to $-\ell_t$.
Using \eqref{eq:Bdecomp}:
\[
\iprod{\W_t - \Wh_t}{\tL_t} = 
  \frac{\eta \ell_t^2 (d-1)^2}{1 + \eta \ell_t (d-1)}
  + (d-1)\frac{\eta \ell_t^2}{1-\ell_t}  \leq
  2 \eta \ell_t^2 \big((d-1)^2 + (d-1)\big)
  \leq 2 \eta \ell_t^2 d^2,
\]
where in the last step we used our assumption on $\eta$ that ensures both 
$|\eta\ell_t| \le \frac{1}{2}$ and $|\eta \pa{d-1} \ell_t| \le \frac{1}{2}$. This concludes 
the proof.
\end{proof}
To conclude the proof of Theorem~\ref{thm:upper_dense} we simply combine
Lemma \ref{lem:mdbound2} and Lemma \ref{lem:deltabound}.

\subsection{Analysis of the sparse sampling method}

The following 
lemma shows that the sparse sampling method achieves a different flavor of data-dependent bound.
\begin{lemma}\label{lem:deltabound-sparse}
Suppose that $\eta \le \frac{1}{2d}$ and $\gamma = \eta d$.
Then, the sparse sampling method guarantees
 \[
  \EEt{\biprod{\W_t - \Wh_{t+1}}{\tL_t}} \le 8 \eta d \norm{\L_t}^2_F
 \]
\end{lemma}
\begin{proof}
Let $\W_t = \sum_i \lambda_i \u_i \u_i\transpose$ be the eigendecomposition of $\W_t$.
Since $\gamma > 0$ the algorithm sample from matrix $\V = (1-\gamma) \W_t + \frac{\gamma}{d} \I$,
which has the same eigenvectors as $\W_t$, and eigenvalues $\mu_i = (1-\gamma) \lambda_i + \gamma/d$.
Sparse sampling draws indices $I$ and $J$ independently from $\boldsymbol{\mu}$. 

Assume the event $I=J=i$
occurred with probability $\mu_i^2$, for which $\tL_t = \frac{\ell_{ii}}{\mu_i^2} \u_i \u_i\transpose$, 
with $\ell_{ii} = \tr(\L_t \u_i \u_i\transpose)$. This means that
$\B_t = \W_t^{1/2} \tL_t \W_t^{1/2} = \frac{\ell_{ii} \lambda_i}{\mu_i^2} \u_i \u_i\transpose$
has single non-zero eigenvalue $b_{t,1} = \frac{\ell_{ii} \lambda_i}{\mu_i^2}$.
Using $(a+b)^2 \geq 4ab$ we have $\mu_i^2 \geq 4(1-\gamma) \gamma \lambda_i /d \geq 2 \gamma \lambda_i /d$,
where we used $\gamma \leq \frac{1}{2}$ which follows from our assumptions. This implies $|b_{t,1}|
\leq \frac{|\ell_{ii}|d}{2 \gamma} \leq \frac{1}{2\eta}$,
which by \eqref{eq:Bdecomp} gives
\begin{equation}
\biprod{\W_t - \Wh_t}{\tL_t} = \frac{\eta b_{t,1}^2}{1 + \eta b_{t,1}} \leq 2 \eta b_{t,1}^2
= 2 \eta \frac{\ell_{ii}^2 \lambda^2_i}{\mu_i^4}.
\label{eq:sparse_bii}
\end{equation}
Now assume event $I = i \neq j = J$ occurred with probability $\mu_i \mu_j$, for which 
$\tL_t = \frac{s \ell_{ij}}{2 \mu_i \mu_j} (\u_i \u_j\transpose + \u_j \u_i\transpose)$
with $\ell_{ij} = \frac{1}{2} \tr(\L_t (\u_i + s \u_j)(\u_i + s \u_j)\transpose)$,
where $s$ is a random sign. This means  
that $\B_t = \frac{s \ell_{ij} \sqrt{\lambda_i\lambda_j}}{2 \mu_i\mu_j} 
 (\u_i \u_j\transpose + \u_j \u_i\transpose)$ has two nonzero eigenvalues equal
 $b_{t,\pm} = \pm s \frac{\ell_{ij} \sqrt{\lambda_i \lambda_j}}{2 \mu_i \mu_j}$.
Using the previously derived bound $\mu_i^2 \geq 2 \gamma \lambda_i /d$, we have
$|b_{t,\pm}| \leq \frac{\sqrt{\lambda_i \lambda_j}}{2 \sqrt{4 \gamma^2 \lambda_i \lambda_j / d^2}}
  = \frac{1}{4 \gamma  / d} = \frac{1}{4 \eta} \leq \frac{1}{2\eta}$, which, similarly as in \eqref{eq:sparse_bii},
  implies
\begin{equation}
\biprod{\W_t - \Wh_t}{\tL_t} \leq 2 \eta b_{t,+}^2 + 2\eta b_{t,-}^2
\leq 2 \eta \frac{\ell_{ij}^2 \lambda_i \lambda_j}{\mu_i^2 \mu_j^2}.
\label{eq:sparse_bij}
\end{equation}
Taking conditional expectation and using \eqref{eq:sparse_bii} and \eqref{eq:sparse_bij} then gives
\[
\mathbb{E}[\biprod{\W_t - \Wh_t}{\tL_t}]
\leq 2 \eta \sum_{ij} \mu_i \mu_j \frac{\mathbb{E}_{s}[\ell_{ii}^2] \lambda_i\lambda_j}{\mu_i^2 \mu_j^2}
= 2 \eta \sum_{ij} \frac{\mathbb{E}_{s}[\ell_{ii}^2] \lambda_i\lambda_j}{\mu_i \mu_j}
\leq 8 \eta \sum_{ij} \mathbb{E}_{s}[\ell_{ii}^2],
\]
where $\mathbb{E}_s[\cdot]$ is the remaining randomization over the sign, and
in the last inequality we used $\frac{\lambda_i}{\mu_i} =
\frac{\lambda_i}{(1-\gamma) \lambda_i + \gamma / d} \leq 
\frac{\lambda_i}{(1-\gamma) \lambda_i} = \frac{1}{1-\gamma} \leq 2$
(because $\gamma \leq \frac{1}{2}$).
For the final step of the proof, let us recall the notation $L_{ij} = \u_i\transpose \L 
\u_j$ and notice that $\ell_{ii}^2 = L_{ii}^2$, whereas for $i \neq j$:
\[
  \mathbb{E}_s[\ell_{ij}^2] = \frac{1}{4} \mathbb{E}_s\big[(L_{ii} + 2 s  L_{ij} + L_{jj})^2 \big]
  = \frac{1}{4} \pa{\pa{L_{ii} + L_{jj}}^2 + 4 L_{ij}^2} \le 
  \frac{1}{4} \pa{2 L_{ii}^2 + 2 L_{jj}^2 + 4 L_{ij}^2}
\]
where in the second equality we used $\mathbb{E}_s[s]=0$ (so that the cross-terms disappear),
while in the last inequality we used $(a+b)^2 \le 2a^2 + 2 b^2$. Thus, we obtained
\[
 \sum_{ij} \EEs{\ell^2_{ij}}{s} \le d \sum_{i} L_{ii}^2 + \sum_{i\neq j} L_{ij}^2 \le 
d \sum_{ij} L_{ij}^2 = d \norm{\L}_F^2,
\]
thus proving the statement of the lemma.
\end{proof}
To conclude the proof of Theorem~\ref{thm:upper_sparse} we simply combine
Lemmas~\ref{lem:mdbound2} and~\ref{lem:deltabound-sparse}.

\subsection{Computational cost}

The total computational cost of the algorithm equipped with dense sampling
is dominated by a rank one update of 
the eigendecomposition of the parameter matrix $\W_t$ in each trial, which can take
$O(d^3)$ time in the worst case. Surprisingly, the computational cost of the
sparse sampling version of the algorithm is only $\tilde{\OO}(d)$.
This is because in each trial $t$, the loss estimate $\tL_t$ is constructed
from up to two eigenvectors of $\W_t$ and thus only the corresponding part of the eigendecomposition
needs to updated. Furthermore, the projection operation only affects the eigenvalues
and can be accomplished by solving a simple line search problem.
The details of the efficient implementation are given in Appendix \ref{appendix:implementation}.
The claimed $\tilde{\OO}(d)$ per-iteration cost of the algorithm is 
without taking into account the time needed to compute the value of the observed loss
(as otherwise reading out the entries of $\L_t$ would already take $O(d^2)$ time).
In other words, we assume that the algorithm
plays with $\w_t$ and the nature computes and communicates the realized loss 
$\ell_t = \tr(\w_t \w_t\transpose \L_t)$ for the learner at no computational cost.
This assumption can actually be verified for several problems of practical interest
(such as the classical applications of phase retrieval),
and helps to separate computational issues related to learning and loss 
computation in other cases.

\section{Discussion}

We conclude by discussing some aspects of our results and possible directions for future work.

\paragraph{Possible extensions.} While we work with real and symmetric matrices throughout the 
paper, it is relatively straightforward to extend our techniques to work with more general losses. 
One important extension is considering complex vector spaces, which naturally arise in applications 
like phase retrieval or quantum information. Fortunately,
our algorithms easily generalize to complex Hermitian matrices, 
essentially by replacing every transposition with 
a Hermitian conjugate, noting that the eigenvalues of Hermitian matrices remain real. 
The analysis can be carried out with obvious modifications \citep{SatyenPhD,quantumstates},
giving the same guarantees on the regret. It would also be interesting to extend our algorithms and 
their analysis the case of asymmetric loss matrices $\L_t \in \mathbb{R}^{m \times n}$, 
where the learner chooses two vectors $\x_t \in \mathbb{R}^n$ and $\y_t \in \mathbb{R}^m$,
and observes loss $\tr(\L_t \x_t \y_t\transpose)$, corresponding to the setup studied by 
\citet{JWWN19}. We note here that extending the basic full-information online PCA formalism is 
possible through a clever embedding of such $m \times n$ matrices into symmetric $(m+n) \times 
(m+n)$ matrices, as shown by \citep{w-ws-07,matrixprediction}. We leave it to future research to 
verify whether such a reduction would also work in the partial-feedback case.

\paragraph{Comparison with continuous exponential weights.} As mentioned in the introduction, the 
bandit PCA problem can be directly formalized as an instance of bandit linear optimization, and one 
can prove regret bounds of $\tilde{\OO}(d^2\sqrt{T})$ by an application of the generic continuous 
Exponential Weights analysis \citep{DHK08,bubeck2014,vanDerHoeven_etal_2018}. However, there are 
two major computational challenges that one needs to face when running this algorithm: sampling the 
density matrices $\W_t$ and the decision vectors $\w_t$, and constructing unbiased estimates for 
the 
losses. Very recently, it has been shown by \citet{PCB18} that one can 
sample and update the exponential-weights distribution in $\OO(d^4)$ time for the decision set we 
consider in this paper, leaving us with the second problem. 
While in principle it is possible to use the generic loss estimator used in the above works (and 
originally proposed by \citealp{McMaBlu04,AweKlein04}), it is unclear if this estimator can 
actually be computed in polynomial time since it involves inverting a linear operator over density 
matrices. 
Indeed, it is not clear if the linear operator itself can be computed in polynomial time, let 
alone its inverse.
In contrast, our algorithms achieve regret bounds of $\wt{\OO}(d^{3/2}\sqrt{T})$ in the 
worst case, and run in $\wt{\OO}(d)$ time when using sparse sampling for loss estimation.

\paragraph{The gap between the upper and lower bounds.} One unsatisfying aspect of our paper is the 
gap of order $\sqrt{d}$ between the upper and lower bounds. Indeed, while 
Algorithm~\ref{alg:banditpca} with sparse sampling guarantees a regret bound of order $d\sqrt{T}$ 
on rank-1 losses, seemingly matching the lower bounds for this case, this upper bound is in fact 
not comparable to the lower bound since the latter is proved for \emph{full-rank} loss matrices. It 
is yet unclear which one of the bounds is tight, and we pose it as an exciting open problem to 
determine the minimax regret in this setup. We believe, however, 
that the upper bounds for our algorithms cannot be improved, and achieving minimax regret would 
require a radically different approach if it is our lower bound that captures the correct 
scaling with $d$.

\paragraph{High-probability bounds.} All our regret bounds proved in the paper hold on expectation. 
It is natural to ask if it is possible to adjust our techniques to yield bounds that hold with high 
probability. Unfortunately, our attempts to prove such bounds were unsuccessful due to a limitation 
common to all known techniques for proving high-probability bounds. Briefly put, all known 
approaches \citep{ACBFS02,BDHKRT08,AB10,BLLRS11,Neu15b} are based on adjusting the unbiased loss 
estimates so that the loss of every action $\v$ is slightly underestimated by a margin of $\beta 
\EEtb{\biprod{\v\v\transpose}{\tL_t^2}}$ for some small $\beta$ of order $T^{-1/2}$ (see, e.g., 
\citealp{AR09} for a general discussion). While it is straightforward to bias our own estimates 
in the same way, this eventually leads to extra terms of order $\beta\EEtb{\biprod{\W_t}{\tL_t^2}}$ 
in the bound, which are impossible to control by a small enough upper bound, as shown in 
Appendix~\ref{sec:MH_Tsallis}. Thus, proving high-probability bounds in our setting seems to 
require a fundamentally new approach, and we pose solving this challenge as another interesting 
problem for future research.

\paragraph{Data-dependent bounds.} Besides a worst-case bound of order $d^{3/2}\sqrt{T}$ on the 
regret, we also provide further guarantees that improve over the above when the loss matrices 
satisfy certain conditions. This raises the question if it is possible to achieve further
improvements under other assumptions on the environment. A particularly interesting question is 
whether or not it is possible to improve our bounds for i.i.d.~loss matrices generated by a 
\emph{spiked covariance model} \citep{Joh01}, corresponding to the most commonly studied setting in 
our primary motivating example of phase retrieval \citep{CSV13,LM15}. Obtaining faster rates for 
this setup would account 
for the discrepancy between the minimax bounds for phase retrieval and those obtained by 
an online-to-batch conversion from our newly proved bounds. We hope that the results provided in 
the present paper will initiate a new line of research on online phase retrieval that will 
eventually yield algorithms that take full advantage of adaptively chosen measurements and 
outperform traditional approaches for phase retrieval.

\bibliography{banditpca}

\newpage

\appendix

\section{Ommitted proofs}
\subsection{The proof of Lemma~\ref{lem:dense_bias}}\label{app:dense_bias}
 For the proof, let us define $L_{ij} = \u_i\transpose \L \u_j$ and note that $\L = \sum_{i,j} 
L_{ij} \u_i \u_j\transpose$. In the case when $B = 1$, we have
\[
\EEcc{\ell \w\w\transpose}{B=1} = \EEcc{\tr(\w\w\transpose \L) \w \w\transpose}{B=1}
= \sum_{i=1}^d \lambda_i \tr(\u_i \u_i\transpose \L) \u_i \u_i\transpose
= \sum_{i=1}^d \lambda_i L_{ii} \u_i \u_i\transpose,
\]
and thus
\[
\EEcc{\tL }{ B=1}
= 2 \W^{-1/2} \left(
\sum_{i=1}^d \lambda_i L_{ii} \u_i \u_i\transpose\right) \W^{-1/2}
= 2 \sum_{i=1}^d L_{ii} \u_i \u_i\transpose,
\]
where we used the fact that $\u_i$ is the eigenvector of $\W$, so $\W^{-1/2} \u_i
= \lambda_i^{-1/2} \u_i$.

When $B=0$, we have:
\begin{align*}
  \EEcc{\tr(\w\w\transpose \L) \w \w\transpose}{B=0}
  &= \EEs{ \tr \left(\sum_{ij} s_i s_j \sqrt{\lambda_i \lambda_j}
    \u_i \u_j\transpose \L \right) \sum_{km} s_k s_m \sqrt{\lambda_k \lambda_m} \u_k \u_m\transpose}{\s} \\
  &= \EEs{\left(\sum_{ij} s_i s_j \sqrt{\lambda_i \lambda_j} L_{ij} \right) 
\left( \sum_{km} s_k s_m \sqrt{\lambda_k \lambda_m} \u_k \u_m\transpose \right)}{\s} \\
&= \sum_{ijkm} \EEs{s_i s_j s_k s_m}{\s} \sqrt{\lambda_i \lambda_j \lambda_k \lambda_m}
 L_{ij} \u_k \u_m\transpose.
\end{align*}
Now, $\EEs{s_i s_j s_k s_m}{\s}$ is zero if one of the indices is a non-duplicate, such as the case
$i \notin \{j,k,m\}$. The four cases where $\EEs{s_i s_j s_k s_m}{\s} = 1$ are the following: (I) 
$i=j=k=m$, (II) $i=j$, $k=m \neq i$, (III) $i=k$, $j=m\neq i$,
(IV) $i=m$, $k=j \neq i$. Considering these cases separately, we get
\begin{align*}
\EEcc{\tr(\w\w\transpose \L) \w \w\transpose}{B=0}
&= \underbrace{\sum_{ij} \lambda_i \lambda_j L_{ii} \u_j \u_j\transpose}_{\text{(I) + (II)}}
+ \underbrace{\sum_{i \neq j} \lambda_i \lambda_j L_{ij} \u_i \u_j\transpose}_{\text{(III)}}
+ \underbrace{\sum_{i \neq j} \lambda_i \lambda_j L_{ij} \u_j \u_i\transpose}_{\text{(IV)}} \\
&= \W \sum_i \lambda_i L_{ii} + 2 \sum_{ij} \lambda_i \lambda_j L_{ij} \u_i \u_j\transpose
- 2 \sum_{i} \lambda_i^2 L_{ii} \u_i \u_i\transpose.
\end{align*}
Multiplying the above with $\W^{-1}$ from both sides gives
\[
\W^{-1} \EEcc{\tr(\w\w\transpose \L) \w \w\transpose}{B=0} \W^{-1}
= \W^{-1} \underbrace{\sum_i \lambda_i L_{ii} }_{\tr(\W \L)}
+ 2 \underbrace{\sum_{ij} \L_{ij} \u_i \u_j\transpose}_{=\L} - 2 \sum_i L_{ii} \u_i \u_i\transpose.
\]
Furthermore, we clearly have $\EEcc{\ell}{B=0} = \tr(\EEcc{\w \w\transpose}{B=0} \L) = \tr(\W \L)$.
Therefore using the definition of $\tL$, we get
\begin{align*}
\EEcc{\tL}{B=0}
&= \W^{-1} \tr(\W \L) + 2 \L - 2 \sum_i L_{ii} \u_i 
\u_i\transpose 
- \W^{-1} \EEcc{\ell }{ B=0}
\\
&= 2 \pa{\L - \sum_i L_{ii} \u_i \u_i\transpose}.
\end{align*}
Putting the two cases concludes the proof as
\[
\EE{\tL} = \frac{1}{2} \EEcc{\tL}{B = 1} + \frac{1}{2} \EEcc{\tL}{B = 0} = \sum_{i=1}^d L_{ii} \u_i 
\u_i\transpose
+ \left(\L - \sum_i L_{ii} \u_i \u_i\transpose \right)
= \L.
\]
\qed

\subsection{The proof of Lemma~\ref{lem:sparse_bias}}\label{app:sparse_bias}
We remind that the loss estimate is constructed as:
\[
\tL ~=~ \left\{
  \begin{array}{ll}
  \frac{\ell}{\lambda_I^2} \u_I \u_I\transpose & \quad \text{when~} I = J,   \\
  \frac{\ell s}{2 \lambda_I \lambda_J} (\u_I \u_J\transpose + \u_J \u_I\transpose) & \quad
  \text{when~} I \neq J.
  \end{array}
  \right.
\]
We check that the estimate of the loss is unbiased. Let $L_{ij} = \u_i\transpose \L \u_j$.
We have:
\begin{align*}
  \EE{\tL}
  &= \underbrace{\sum_{i} \lambda_i^2 \tr(\u_i \u_i\transpose \L) \frac{1}{\lambda_i^2} \u_i \u_i\transpose}_{\text{when~} I=J} \\
  &\qquad + \underbrace{\sum_{i \neq j} \lambda_i \lambda_j \EEs{
    \tr\left(\frac{1}{2}(\u_i + s \u_j)
    (\u_i + s \u_j)\transpose \L \right) \frac{s}{2 \lambda_i \lambda_j} 
    (\u_i \u_j\transpose + \u_j \u_i\transpose)}{s}}_{\text{when~} I \neq J} \\
  &= \sum_i L_{ii} \u_i \u_i\transpose
  + \frac{1}{4} \sum_{i \neq j} (L_{ii} + L_{jj}) 
  \underbrace{\EEs{s}{s}}_{=0} (\u_i \u_j\transpose + \u_j \u_i\transpose) 
  + \frac{1}{2} \sum_{i \neq j} L_{ij} (\u_i \u_j\transpose + \u_j \u_i\transpose)  \\
  &= \sum_{ij} L_{ij} \u_i \u_j\transpose = \L,
\end{align*}
where in the second inequality we used the fact that $s^2=1$.
\qed

\subsection{The proof of Lemma~\ref{lem:mdbound2}}
\label{app:mdbound}
We start with the well-known result regarding the regret of mirror 
descent \citep[see, e.g.,][]{RakhlinLN}. We include the simple proof in 
for completeness.
\begin{lemma} \label{lem:mdbound}
For any $\U\in\mathcal{S}$, the following inequality holds:
 \begin{equation*}
  \sum_{t=1}^T \biprod{\W_t - \U}{\tL_t} \le \frac{D_R(\U\|\W_1)}{\eta} + \sum_{t=1}^T \biprod{\W_t 
- \Wh_{t+1}}{\tL_t}.
 \end{equation*}
\end{lemma}
\begin{proof}
We start from the following well-known identity\footnote{This easily proven result is sometimes 
called the ``three-points identity''.} that holds for for any 
three SPSD matrices $\U, \V, \W$: 
\[
  D_R(\U \| \V) + D_R(\V \| \W) = D_R(\U \| \W) 
  + \iprod{\U-\V}{\nabla R(\W) - \nabla R(\V)}.
\]
Taking $\W = \W_t$ and $\V = \Wh_{t+1}$ and using that $D_R(\V \| \W) \ge 0$ gives
\[
  D_R(\U \| \Wh_{t+1}) \leq 
   D_R(\U \| \W_t) + \eta \iprod{\U - \Wh_{t+1}}{\tL_t}.
\]
Since $D(\U \| \Wh_{t+1}) \geq D(\U \| \W_{t+1})$ by the Generalized Pythagorean
Inequality, we get
\[
  \eta \iprod{\Wh_{t+1} - \U}{\tL_t} \leq D_R(\U \| \W_t) - D_R(\U \| \W_{t+1}).
\]
Reordering and adding $\big\langle\W_t,\tL_t\big\rangle$ to both sides gives
\[
  \iprod{\W_t - \U}{\tL_t} \leq \iprod{\W_t - \Wh_{t+1}}{\tL_t}
  + \frac{1}{\eta} D_R(\U \| \W_t) - \frac{1}{\eta} D_R(\U \| \W_{t+1}).
\]
Summing up for all $t$ and noticing that $D_R(\U \| \W_{T+1}) \ge 0$ concludes the proof.
\end{proof}

\begin{proof}{(of Lemma \ref{lem:mdbound2}).} 
We start with relating the quantity on the left-hand side of statement in Lemma \ref{lem:mdbound}
to the regret of the algorithm. To this end, observe that the unbiasedness 
of $\tL_t$ and the conditional independence of $\tL_t$ on $\W_t$ ensures that 
\[
 (1-\gamma) \EEt{\biprod{\W_t}{\tL_t}} = \biprod{(1-\gamma)\W_t}{\L_t} = \EEt{\biprod{\w_t\w_t}{\L_t}} 
 - \frac{\gamma}{d} \iprod{\I}{\L_t},
\]
where we also used the fact that $\w_t$ is sampled so that $\EEt{\w_t\w_t} = (1-\gamma) \W_t + 
\frac{\gamma}{d} \I$ is satisfied. Similarly, for any fixed $\U$
it holds $\EEt{\biprod{\U}{\tL_t}} = \biprod{\U}{\L_t}$. Using these relation results in
\begin{equation}
 \EEt{\biprod{\w_t\w_t - \U}{\L_t}} = (1-\gamma) \EEt{\biprod{\W_t - \U}{\tL_t}}
 + \gamma \iprod{\frac{\I}{d} - \U}{\L_t}.
\label{eq:some_intermediate_stuff}
\end{equation}
Since $\L_t$ has spectral norm bounded by $1$, the last term on the 
right-hand side can be bounded by:
\[
  \iprod{\frac{\I}{d} - \U}{\L_t}
  \leq \left\| \frac{\I}{d} - \U \right\|_1 \|\L_t \|_{\infty}
  \leq \tr\left(\frac{\I}{d}\right) + \tr(\U) = 2
\]
Using the above bound in \eqref{eq:some_intermediate_stuff}, summing over trials
and taking marginal expectation on both sides gives:
\begin{align}
\sum_{t=1}^T \EE{\biprod{\w_t\w_t - \U}{\L_t}} &\le (1-\gamma) \sum_{t=1}^T \EE{\biprod{\W_t - \U}{\tL_t}} + 2 \gamma T \nonumber \\
&\le (1-\gamma) \frac{D_R(\U\|\W_1)}{\eta} + (1-\gamma) \sum_{t=1}^T \EE{\biprod{\W_t 
- \Wh_{t+1}}{\tL_t}} + 2 \gamma T,
\label{eq:yet_another_bound}
\end{align}
where the second inequality is from Lemma \ref{lem:mdbound}.
One minor challenge is that the first term on the right-hand side of \eqref{eq:yet_another_bound}
is infinite for a  ``pure'' comparator $\u\u\transpose$.
To deal with this issues, for any $\U$
define the smoothed comparator $\hU = (1-\theta) \U + \frac{\theta}{d} \I$ for some
$\theta \in [0,1]$. Using $\W_1 = \frac{1}{d} \I$, we have:
\begin{align*}
D_R(\hU\|\W_1) &= \log \det \left(\frac{\I}{d}\right) - \log \det \left((1-\theta) \U + \frac{\theta}{d}\I\right)
+ d \tr(\hU) - d \\
&\leq \log \det \left(\frac{\I}{d}\right) - \log \det \left(\frac{\theta}{d}\I\right)
= d \log(1/\theta).
\end{align*}
Using \eqref{eq:yet_another_bound} with the smoothed comparator $\hU$ gives:
\[
\sum_{t=1}^T \EE{\biprod{\w_t\w_t - \hU}{\L_t}} 
\leq
(1-\gamma) \frac{d \log(1/\theta)}{\eta} + (1-\gamma) \sum_{t=1}^T \EE{\biprod{\W_t 
- \Wh_{t+1}}{\tL_t}} + 2 \gamma T.
\]
Now, since: 
\[
\biprod{\w_t\w_t - \hU}{\L_t}
= \biprod{\w_t\w_t - \U}{\L_t} + \theta \iprod{\frac{\I}{d} - \U}{\L_t}
\geq \biprod{\w_t\w_t - \U}{\L_t} - 2 \theta
\]
(where we used a bound on the spectral norm of $\L_t$), setting $\theta = 1/T$ gives:
\[
\sum_{t=1}^T \EE{\biprod{\w_t\w_t - \U}{\L_t}} 
\leq
\frac{d \log T}{\eta} + (1-\gamma) \sum_{t=1}^T \EE{\biprod{\W_t 
- \Wh_{t+1}}{\tL_t}} + 2 \gamma T + 2.
\]
\end{proof}

\subsection{The proof of Corollary~\ref{cor:upper_dense_lstar}}
\label{appendix:corollary_dense}
From the non-negativity and boundedness of the loss matrices it follows that 
$\ell_t = \tr(\w_t \w_t\transpose \L_t) \in [0,1]$, which implies $\ell_t^2 \leq \ell_t$. 
Let $L_T^* = \min_{\u:\norm{\u}=1}
\mathbb{E}\big[\sum_{t=1}^T \iprod{\u\u\transpose}{\L_t}\big] \leq \overline{L}_T^*$
be the expected loss of the optimal comparator, and let
$\widehat{L}_T = \mathbb{E}\big[\sum_{t=1}^T \ell_t\big]$ be the algorithm's 
expected cumulative loss.
By Theorem \ref{thm:upper_dense} (using $\ell_t^2 \leq \ell_t$):
\[
 \regret_T = \widehat{L}_T - L_T^* \le 
 \frac{d \log T}{\eta} + \eta (d^2+1) \widehat{L}_T + 2
\]
which can be reordered to imply the bound for $\eta < 1/(d^2+1)$:
\[
 (1-\eta(d^2+1)) \regret_T \le \frac{d\log T}{\eta} + \eta (d^2+1) \overline{L}_T^* + 2 
\]
Thus, if $\overline{L}_T^* \ge 16 d^3 \log T$, we can set $\eta = 
\sqrt{\frac{\log T}{d\overline{L}^*_T}} \leq \frac{1}{2(d^2+1)}$ and obtain the bound
\[
 \regret_T \le 6 d^{3/2} \sqrt{\overline{L}_T^* \log T} + 2 
\]
Otherwise, we can set $\eta = 1/(2(d^2+1))$ and get
\[
 \regret_T \le 24 d^{3} \log T + 4.
\]

\qed

\subsection{The proof of Theorem~\ref{thm:lower}}
\label{appendix:lower_bound}
In this section, we provide the proof of our lower bound presented in Theorem~\ref{thm:lower}. Our overall proof strategy is based on the classical 
recipe for proving worst-case lower bounds in bandit problems---see, e.g., Theorem~5.1 in \citet{ACBFS02} or Theorem~6.11 in \citet{CBL06}. Specifically, we 
will construct a stochastic adversary and show a lower bound on the regret of any deterministic learning algorithm on this instance, which implies a lower 
bound 
on randomized algorithms on any problem instance by Yao's minimax principle \citep{Yao77}. The lower bound for deterministic strategies will be proven using 
classic information-theoretic arguments. 
The adversary's strategy will be to draw $\u \in \real^d$ uniformly at random from the unit sphere before the first round of the game, and play with loss 
matrices of the form
\[
 \L_t = Z_t \I - \epsilon \u\u\transpose,
\]
where $Z_t \sim N(0,1)$ and $\epsilon \in [0,1]$ is a tuning parameter that will be chosen later. 
An important feature of this construction is that it keeps the signal-to-noise ratio small by correlating the losses of each 
action through the global loss $Z_t$ suffered by each action.  This technique is inspired by the work of \citet{CHK17}, and is crucially important for 
obtaining 
a linear scaling with $d$ in our lower bound.

Note that spectral norm of $\L_t$ is not bounded, but has sub-Gaussian tails. This, however, comes (almost) without
loss of generality: by Theorem 7 from
\citet{pmlr-v40-Shamir15} the lower bound for such sub-Gaussian losses can be converted into a lower bound on the bounded losses
at a cost of mere $\sqrt{\log T}$.

Define $\EEu{\cdot} = \EEcc{\cdot}{\u}$ as the expectation conditioned on $\u$ and 
 $\EEo{\cdot}$ as the total expectation when $\epsilon=0$. 
Observe that we have $\EEu{\L_t} = - \epsilon \u\u\transpose$, so we can bound the
loss of the comparator as
\[
\EE{\inf_{\U \colon \tr(\U)=1}\sum_{t=1}^T \tr(\U \L_t)}
 \le \EE{\EEu{ \sum_{t=1}^T \tr(\u \u\transpose \L_t) }}
 = -\epsilon T,
\]
where we defined $\EEu{\cdot} = \EEcc{\cdot}{\u}$ as the expectation conditioned on $\u$. 
On the other hand, the expected loss of the learner is given by
\[
\EE{\sum_{t=1}^T \tr(\w_t\w_t\transpose \L_t)}
= -\epsilon \EE{\EEu{\sum_{t=1}^T \tr(\w_t\w_t\transpose \u \u\transpose)}},
\]
so the regret can be lower-bounded as
\[
 \mathrm{regret}_T
 \geq \epsilon T - \epsilon \EE { \EEu{ \sum_{t=1}^T \tr(\w_t\w_t\transpose \u \u\transpose) }}.
\]
Now note that
\[
\EE {\EEo{ \sum_{t=1}^T \tr(\w_t\w_t\transpose \u \u\transpose) }}
= \EEo { \sum_{t=1}^T \tr(\w_t\w_t\transpose \EE{\u \u\transpose}) }
= \EEo { \sum_{t=1}^T \tr\left(\w_t\w_t\transpose \frac{\I}{d}\right) }
= \frac{T}{d},
\]
where we used the fact that $\u$ is independent of $\w_1,\ldots,\w_T$ when $\epsilon = 0$, and that $\EE{\u \u\transpose} = \frac{\I}{d}$ when $\u$ is
uniformly distributed over the unit sphere. Thus, the regret can be rewritten as
\[
\mathrm{regret}_T \geq \epsilon T \left(1 - \frac{1}{d}\right) - \epsilon 
 \mathbb{E}\Biggl[\underbrace{\EEu { \sum_{t=1}^T \tr(\w_t\w_t\transpose \u \u\transpose)}
 - \EEo{\sum_{t=1}^T \tr(\w_t\w_t\transpose \u \u\transpose)}}_{=\Delta_{\u}} \Biggr],
\]
which leaves us with the problem of upper-bounding $\Delta_{\u}$.

To this end, let $\ell^T = (\ell_1,\ldots,\ell_T)$ be the sequence of losses generated by the deterministic strategy,
and let $p_{\u}(\ell^T)$ denote the density of $\ell^T$ conditionally on $\u$. 
Notice that $\w_t$ is completely determined by $g^{t-1}$. 
Furthermore,
let $p_0(g^T)$ denote the corresponding density of $\ell^T$ when $\epsilon=0$, implying that $\L_t = Z_t \I$ for all $t$.
Defining $F(g^T) = \sum_{t=1}^T \tr(\w_t \w_t\transpose \u \u\transpose)$, we can write $\Delta_{\u}$ as
\begin{align*}
 \Delta_{\u} &= \int F(\ell^T) \pa{p_{\u}(\ell^T) - p_0(\ell^T) } \dif \ell^T
 \leq \int_{p_{\u}(\ell^T) \geq p_0(\ell^T)} F(\ell^T) \pa{p_{\u}(\ell^T) - p_0(\ell^T)} \dif \ell^T \\
 &\leq T \int_{p_{\u}(\ell^T) \geq p_0(\ell^T)} \pa{p_{\u}(\ell^T) - p_0(\ell^T)} \dif \ell^T
 \leq T D_{\mathrm{TV}}(p_0 \| p_{\u})
 \leq T \sqrt{\frac{1}{2} D_{\mathrm{KL}}(p_0 \| p_{\u})},
\end{align*}
where $D_{\mathrm{TV}}(\cdot \| \cdot)$ and $D_{\mathrm{KL}}(\cdot \| \cdot)$ denote, respectively, the total variation distance
and the Kullback-Leibler (KL) divergence between two distributions, and the last step uses Pinsker's inequality, while
the second inequality uses $F(g^T) = \sum_t (\w_t\transpose\u)^2 \leq T$.
By the chain rule for the KL divergence, we have
\[
D_{\mathrm{KL}}(p_0 \| p_{\u}) 
= \sum_{t=1}^T \EEo{ D_{\mathrm{KL}}\pa{p_0(\ell_t| \ell^{t-1}) \middle\| p_{\u}(\ell_t | \ell^{t-1})}}
\]
Now, the loss in round $t$ can be written as 
$\ell_t = \w_t\transpose \L_t \w_t$. By the definition of $\L_t$, the conditional distribution of $\ell_t$ is 
Gaussian with unit variance under both $p_{\u}$ and $p_0$: $\ell_t = Z_t - \epsilon (\w_t\transpose \u)^2 \sim N(-\epsilon(\w_t\transpose \u)^2, 1)$ under $p_{\u}$ and 
$\ell_t = Z_t \sim N(0,1)$ under $p_0$.
Thus, the conditional KL divergence between the two distributions can be written as
\[
D_{\mathrm{KL}}\pa{p_0(g_t| g^{t-1}) \middle\| p_{\u}(g_t | g^{t-1})}
= \frac{1}{2} \epsilon^2 (\w_t\transpose \u)^4,
\]
which implies
\[
 \Delta_{\u} \leq \frac{T}{2} \epsilon \sqrt{\sum_{t=1}^T \EEo{ (\w_t\transpose \u)^4 }}.
\]

In order to bound $\EE{\Delta_{\u}}$, we use Jensen's inequality $\EE{\sqrt{\cdot}} \leq \sqrt{\EE{\cdot}}$ to write
\[
 \EE{\Delta_{\u}}
 \leq \frac{T}{2} \epsilon \sqrt{\sum_{t=1}^T \EE{\EEo{(\w_t\transpose \u)^4}}}
 = \frac{T}{2} \epsilon \sqrt{\sum_{t=1}^T \EEo{\EE{(\w_t\transpose \u)^4}}},
\]
where in the last step we swapped the order of expectations as $\u$ is independent of $\ell_1,\ldots,\ell_T$ under $p_0$.
Since $\u$ is distributed uniformly over the unit sphere, $\pa{\w_t\transpose\u}$ has the same distribution as $u_1$.
Using the fact that $u_1^2 \sim \mathrm{Beta}\left(\frac{1}{2},\frac{d-1}{2}\right)$ \citep{devroye:1986}, this implies:
\[
 \EE{(\w_t\transpose \u)^4} = \EE{u_1^4} = \frac{3}{d(d+2)},
\]
Thus, we arrive to the bound
\[
 \EE{\Delta_{\u}} \leq \frac{T}{2} \epsilon \sqrt{T \frac{3}{d(d+2)}}
 \leq \frac{T^{3/2}}{d} \epsilon,
\]
which, put together with the previous calculations, eventually gives
\[
  \mathrm{regret}_T \geq \epsilon T \left(1 - \frac{1}{d}\right) - \epsilon^2 
  \frac{T^{3/2}}{d}.
\]
Bounding $1-\frac 1d \ge \frac 12$ and setting  $\epsilon = dT^{-1/2}/4$ 
gives $\mathrm{regret}_T = \Omega(d \sqrt{T})$, which by aforementioned
Theorem 7 from \citet{pmlr-v40-Shamir15} implies the claim in the theorem. \qed

\section{Efficient implementation of the update}
\label{appendix:implementation}
In this section, we give details on the efficient implementation of the 
mirror descent update \eqref{eq:MD_update}:
\[
\begin{array}{rl}
  \text{(update step)} &\displaystyle \Wh_{t+1} = \argmin_{\W} 
  \left\{D_R(\W \| \W_t) + \eta \tr(\W \tL_t) \right\},   \\[2mm]
  \text{(projection step)} & \displaystyle \W_{t+1} 
  = \argmin_{\W \in \mathcal{W}} D_R(\W \| \Wh_{t+1}),
\end{array}
\]
with the Bregman divergence induced by the negative log-determinant regularizer:
\[
  D_R(\W \| \U) = \tr(\U^{-1} \W) - \log \frac{\det(\W)}{\det(\U)} - d
\]
As we will show, the algorithm runs in time $\tilde{\OO}(d)$ per trial for sparse
sampling method, and in time $\tilde{\OO}(d^3)$ for dense sampling method.
In what follows, we assume that the eigenvalue decomposition
$\W_t = \sum_i \mu_i \u_i \u_i\transpose$ is given at the beginning of trial $t$,
where $\{\u_i\}_{i=1}^d$ are the eigenvectors, 
and $\{\lambda_i\}_{i=1}^d$ are the eigenvalues of $\W_t$ (sorted in a decreasing order),
and we dropped the trial index for the sake of clarity. 
The eigenvalues of $\W_t$ then get mixed with a uniform distribution:
\[
  \lambda_i = (1-\gamma) \mu_i + \gamma \frac{1}{d}, \qquad i=1,\ldots,d
\]
(with $\gamma=0$ for dense sampling) and are used to sample the action of the algorithm.

\subsection{The update step}
We have shown in Section \ref{sec:analysis} that the unprojected
solution is given by \eqref{eq:unprojected}:
\[
\Wh_{t+1} = \W_t^{1/2} \left(\I + \eta \B_t\right)^{-1} \W_t^{1/2},
\qquad \text{where~~} \B_t = \W_t^{1/2} \tL_t \W_t^{1/2}.
\]

\paragraph{Sparse sampling.}
Two indices $I,J \in \{1,\ldots,d\}$ are independently sampled from the same distribution satisfying $\PP{J = i} = \PP{I = i} = 
\lambda_i$ (which takes negligible $\OO(\log d)$ time). 

When $I=J$, the algorithm plays with $\w = \u_I$, receives $\ell_t$,
and the loss estimate is given by
$\tL = \frac{\ell}{\lambda_I^2} \u_I \u_I\transpose$. As $\u_I$ is one of the
eigenvectors of $\W_t$, we obtain $\B_t = \ell_t \frac{\mu_I}{\lambda_I^2} \u_I \u_I\transpose$.
This means that $\W_t$ and $\I + \eta \B_t$ commute so that $\Wh_{t+1}$ has the same eigensystem
as $\W_t$ and it only amounts to computing the eigenvalues $(\mu_1',\ldots,\mu_d')$ 
of $\Wh_{t+1}$, which
are given by:
\[
\mu_i' = \left\{
  \begin{array}{ll}
    \mu_i & \quad \text{for~} i \neq I, \\
    \frac{1}{1 + \eta \ell_t \mu_I / \lambda_I^2} \mu_I 
      & \quad \text{for~} i = I.
  \end{array}
\right.
\]
As the eigenvectors do not change, and only one eigenvalue is updated,
the eigendecomposition of $\Wh_{t+1}$ is updated in time $\OO(1)$.

When $I \neq J$, the algorithm plays with $\w = \frac{1}{\sqrt{2}}(\u_I + s \u_J)$,
where $s \in \{-1,1\}$ is a random sign.
The loss estimate is $\tL = \frac{s \ell}{2 \lambda_I \lambda_J} (\u_I \u_J\transpose + \u_J \u_I\transpose)$,
which gives $\B_t = \frac{\sqrt{\mu_I \mu_J} s \ell}{2 \lambda_I \lambda_J} (\u_I \u_J\transpose + \u_J \u_I\transpose).$
To simplify notation, we denote:
\[
\I + \eta \B_t = \I + \beta (\u_I \u_J\transpose + \u_J \u_I\transpose), \qquad \text{where~} \beta = \frac{\eta \sqrt{\mu_I \mu_J} s \ell}{2 \lambda_I \lambda_J}.
\]
Due to rank-two representation of $\B_t$, which involves only two eigenvectors of $\W_t$, the eigenvectors
and eigenvalues of $\Wh_{t+1}$ will be the same as for $\W_t$, except for those associated with drawn indices $I$ and $J$.
Specifically, it can be verified by a direct computation that the inverse of 
$\I + \eta \B_t$ is given by:
\[
\left(\I + \beta (\u_I \u_J\transpose + \u_J \u_I\transpose)\right)^{-1}
= \I + \frac{\beta^2}{1-\beta^2} (\u_I \u_I\transpose + \u_J \u_J\transpose)
- \frac{\beta}{1-\beta^2} (\u_I \u_J\transpose + \u_J \u_I\transpose).
\]
Multiplying the above from both sides by $\W_t^{1/2}$ gives:
\begin{align*}
\Wh_{t+1} &= \W_t + \frac{\beta^2}{1-\beta^2} (\mu_I \u_I\u_I\transpose + \mu_J \u_J\u_J\transpose)
- \frac{\beta\sqrt{\mu_I \mu_J}}{1-\beta^2} ( \u_I \u_J\transpose + \u_J \u_I\transpose) \\
&= \sum_{i \notin \{I,J\}} \mu_i \u_i \u_i\transpose 
~+~ \frac{1}{1-\beta^2} \Big(\mu_I \u_I \u_I\transpose + \mu_J \u_J \u_J\transpose
- \beta\sqrt{\mu_I \mu_J} ( \u_I \u_J\transpose + \u_J \u_I\transpose) \Big).
\end{align*}
As the term in parentheses on the right-hand side 
only concerns the subspace spanned by $\u_I$ and $\u_J$, 
$\Wh_{t+1}$ has eigendecomposition $\Wh_{t+1} = \sum_{i \notin \{I,J\}} \mu_i \u_i \u_i^\top
+ \mu_+ \u_1\u_+\transpose + \mu_- \u_- \u_-\transpose$, where $\u_+$ and $\u_-$
are linear combinations of $\u_I$ and $\u_J$. Specifically:
\begin{align*}
\mu_{\pm} &= \frac{\mu_I + \mu_J \pm \sqrt{(\mu_I - \mu_J)^2 + 4\mu_I \mu_J \beta^2}}{2(1-\beta^2)}, \\
\u_{\pm} &= \frac{-\beta \sqrt{\mu_I \mu_J} \u_I
+ (\mu_{\pm}(1-\beta^2) - \mu_I) \u_J }{\sqrt{\beta^2 \mu_I \mu_J + (\mu_{\pm}(1-\beta^2) - \mu_I)^2}}.
\end{align*}
Thus, we only need to update two eigenvalues and their corresponding eigenvectors, which can be done in $O(d)$.

\paragraph{Dense sampling.}
For the ``on-diagonal'' sampling, $\B_t = 2 \ell_t \u_i \u_i\transpose$,
where $\u_i$ is one of the eigenvectors of $\W_t$. This means that
$\W_t$ and $\I + \eta \B_t$ commute so that $\Wh_{t+1}$ has the same eigensystem
as $\W_t$ and it only amounts to computing the eigenvalues $(\lambda_1',\ldots,\lambda_d')$ 
of $\Wh_{t+1}$, which
are given by:
\[
\lambda_j' = \left\{
  \begin{array}{ll}
    \lambda_j & \quad \text{for~} j \neq i \\
    \frac{1}{1 + 2 \eta \ell_t} \lambda_i 
      & \quad \text{for~} j = i
  \end{array}
\right.
\]
For the ``off-diagonal'' sampling, we have
$\B_t = \ell_t (\v_t \v_t\transpose - \I)$ where 
$\v_t = \sum_{i=1}^d s_i \u_i$.
Using Sherman-Morrison formula we can invert $\I + \eta \B_t = \I(1 - \eta \ell_t)
+ \eta \ell_t \v_t \v_t\transpose$ to get:
\[
  \Wh_{t+1}
  = \frac{1}{1 - \eta \ell_t} \W_t^{1/2} \left(
  \I - \frac{\eta \ell_t \v_t \v_t\transpose}{1 + \eta (d-1)\ell_t} 
  \right) \W_t^{1/2},
\]
where we used $\v_t\transpose \v_t = \|\v_t\|^2 = d$.
To calculate the eigendecomposition of $\Wh_{t+1}$,
we rewrite the expression above as:
\[
  \Wh_{t+1} = \frac{1}{1 - \eta \ell_t} \U \bigg(
  \underbrace{\boldsymbol{\Lambda} - 
  \frac{\eta \ell_t \widetilde{\v}_t \widetilde{\v}_t\transpose}
  {1 + \eta(d-1) \ell_t}}_{\A} \bigg) \U\transpose,
\]
where $\U = [\u_1,\ldots,\u_d]$ stores the eigenvectors of $\W_t$ as columns, 
$\boldsymbol{\Lambda} = \mathrm{diag}(\lambda_1,\ldots,\lambda_d)$, and $\widetilde{\v}_t
= \sum_{i=1}^d s_i \lambda_i^{1/2} \e_i$, with $\e_i$ being the $i$-th unit vector
(with $i$-th coordinate equal to $1$ and remaining coordinates equal to $0$).
Thus, we first calculate the eigendecomposition of $\A$, and then multiply the resulting
eigenvectors by $\U$ to get the eigendecomposition of $\Wh_{t+1}$. We note that $\A$
is a rank-one update of the diagonal matrix, which eigendecomposition can be calculated
in $\OO(d^2)$ \citep{eigdec}. However, the multiplication of eigenvectors of $\A$ by $\U$
still takes $\OO(d^3)$, which is also the dominating cost of the whole update with dense sampling.

\subsection{The projection step}
The projection step reduces to solving:
\begin{equation}
  \W_t = \argmin_{\W \in \mathcal{W}}
    \tr\left(\Wh_{t+1}^{-1} \W\right) - \log \det(\W).
\label{eq:projection_to_solve}
\end{equation}
We first argue that $\W_t$ and $\Wh_{t+1}$ have the same eigenvectors, and the projection
only affects the eigenvalues.
Note that $\det(\W)$ only depends on the eigenvalues of $\W$ and not on its eigenvectors.
Furthermore, for any symmetric matrices $\A$ and $\B$,
$\tr(\A \B) \geq \sum_{i=1}^d \lambda_{d-i}(\A) \lambda_i(\B)$, where
$\lambda_i(\A), \lambda_i(\B)$ denote the eigenvalues of $\A$ and $\B$, respectively,
sorted in a decreasing order \citep[][Fact 5.12.4]{matrixbook}. This means that
if we let $\bnu = (\nu_1,\ldots,\nu_d)$ and $\bmu = (\mu_1,\ldots,\mu_d)$ denote
the eigenvalues of $\Wh_{t+1}$ and $\W$, respectively, sorted in a decreasing order,
then $\tr(\Wh_{t+1}^{-1} \W) \geq \sum_{i=1}^d \nu_i^{-1} \mu_i$, with the equality
if and only if the eigenvectors of $\Wh_{t+1}$ and $\W^{-1}$ are the same. This means that
if we fix the eigenvalues of $\W$, then the right-hand side 
of \eqref{eq:projection_to_solve}
is minimized by $\Wh_{t+1}$ and $\W_t$ sharing their eigenvectors.

Thus, the projection can be reduced to finding the eigenvalues $\bmu$ of $\W_t$:
\[
  \bmu = \argmin_{\bmu \in \mathcal{M}} \sum_{i=1}^d \frac{\mu_i}{\nu_i} - \log \mu_i,
  \qquad \mathcal{M} = \{\bmu \colon \mu_1 \geq \mu_2 \geq \ldots \geq \mu_1 \geq 0, \sum_i \mu_i = 1\}
\]
In fact, the first constraint in $\mathcal{M}$ is redundant, 
as the positivity of $\mu_i$ is implied by
the domain of the logarithmic function, and if $\mu_i < \mu_{i+1}$ for any $i$ such
that $\nu_i > \nu_{i+1}$, then
it is straightforward to see that swapping the values of $\mu_i$ and $\mu_{i+1}$ 
decreases the objective function. Taking the derivative of the right-hand side and
incorporating the constraint $\sum_i \mu_i = 1$ by introducing the Lagrange multiplier
$\theta$ gives for any $i=1,\ldots,d$:
\[
  \nu_i^{-1} - \mu_i^{-1} + \theta = 0
  \qquad \Longrightarrow \qquad \mu_i = \frac{1}{\nu_i^{-1} + \theta}
\]
The value of $\theta$ satisfying $\sum_i \mu_i = 1$ can now be easily obtained by a root-find algorithm, 
e.g., by the Newton method (alternatively, we can cast the problem as one-dimensional minimization of a convex
function $f(\theta) = -\sum_i \log(\mu_i^{-1} + \theta) - \theta$).
As the time complexity of a single iteration is $\OO(d)$ and the number of iterations required to achieve
error of order $\epsilon$ is at most $\OO(\log \epsilon^{-1})$, the total runtime is $\OO(d \log \epsilon^{-1})$.
Since the errors may generally accumulate over time we need to set $\epsilon^{-1}$ to scale polynomially with $T$
(so that the total error at the end of the game will still be negligible),
which means that the runtime is of order $\OO(d \log T) = \tilde{\OO}(d)$.

\section{Matrix Hedge and Tsallis regularizers}
\label{sec:MH_Tsallis}
In this section we explain some technical difficulties that we faced while attempting to analyze 
variants of our algorithm based on the regularization functions most commonly used in multi-armed 
bandit problems: Tsallis entropies and the Shannon entropy (known as the quantum entropy function 
in the matrix case). This section is not to be regarded as a counterexample against any of these 
algorithms, but rather a summary of semi-formal arguments suggesting that the algorithms derived 
from these regularization functions may fail to give near-optimal performance guarantees. In 
fact, we believe that obstacles we outline here might be impossible to overcome.

\paragraph{Matrix Hedge.} Consider the online mirror descent algorithm~\eqref{eq:MD_update} 
equipped with the quantum negative entropy regularizer $R(\U) = \tr(\U \log \U)$ and any unbiased 
loss estimate $\hL_t$ satisfying $|\eta \tL_t| = \OO(1)$ (which can be achieved by an 
appropriate amount of forced exploration, without loss of generality). This corresponds to a 
straightforward bandit variant of the algorithm known as Matrix Hedge (MH) 
\citep{meg,AroraHazanKale,WarmuthKuzmin2008}.
Following standard derivations (e.g., by \citealp{matrixprediction}), one can easily show an upper 
bound on the regret of the form 
\[
\mathrm{regret}_T
\leq \frac{\ln d}{\eta} + c_1 \eta 
\sum_{t=1}^T \EE{\tr \left(\W_t \tL_t^2\right)} + c_2,
\]
for some constants $c_1$ and $c_2$. What is thus left is
to bound the ``variance'' terms $\EEb{\tr \bpa{\W_t \tL_t^2}}$ by a (possibly dimension-dependent) 
constant for all $t$, and tune the learning rate appropriately.
While this is easily accomplished in the standard multi-armed bandit setup by exploiting the 
properties of importance-weighted loss estimates, controlling this term becomes much harder in the 
matrix case.

We formally show below that the variance term described above cannot be upper bounded by \emph{any} 
constant for \emph{any} natural choice of unbiased loss estimator. In what follows, we drop the 
time index $t$ for the sake of clarity.
We assume the loss estimate has a general form $\tL = \ell \H$, where
$\ell = \iprod{\L}{\w \w\transpose}$ is the observed loss and $\H$ is some matrix that does 
not depend on $\ell$ (but will depend on the action $\w\w\transpose$ of the learner). Notably, this 
class of loss estimators include all known unbiased loss estimators for linear bandits.
We will show that when $\L \succeq \alpha \I$, then $\EE{\tr 
\left(\W \tL^2\right)} \geq \frac{c}{\lambda_{\min}(\W)}$,
where $\alpha$ and $c$ are some positive constants, and $\lambda_{\min}(\W)$ is the smallest 
eigenvalue of $\W$. This clearly implies that one cannot upper bound the variance terms by a 
constant, since there is no way in general to lower bound $\lambda_{\min}(\W)$ by a constant 
independent of $T$.

To make the analysis as simple as possible, consider the case $d=2$ and (without loss of generality) 
assume $\W = \mathrm{diag}(\lambda_1,\lambda_2)$. Let $\w = (w_1,w_2)$ be the action of the algorithm.
Since $\EE{\w \w\transpose} = \W$ we have:
\[
  \EE{w_1^2} = \lambda_1, \quad \EE{w_2^2} = \lambda_2.
\]
Furthermore the observed loss is given by:
\[
  \ell = \tr(\w \w\transpose \L) 
  = \w\transpose \L \w = w_1^2 L_{11} + w_2^2 L_{22} + w_1 w_2 L_{12},
\]
where $L_{ij}$ are the entries of $\L$.
The condition $\EEb{\tL} = \L$ thus implies: 
\[
  \EE{(w_1^2 L_{11} + 2 w_1 w_2 L_{12} + w_2^2 L_{22}) H_{12}}
  = L_{12},
\]
where $H_{12}$ is the off-diagonal entry of $\H$. The right-hand
side of the above does not depend on $L_{11}$ and $L_{22}$, and since these numbers can be arbitrarily chosen
by the adversary,
the left-hand side cannot depend on them either. This means that
$\EE{w_1^2 H_{12}} = \EE{w_2^2 H_{12}} = 0$, and 
$\EE{w_1 w_2 H_{12}} = \frac{1}{2}$. From the last expression we get:
\[
  \frac{1}{2} =  \EE{w_1 w_2 H_{12}} \leq \sqrt{\EE{w_1^2 w_2^2}}
  \sqrt{\EE{H_{12}^2}}
  \quad \Longrightarrow \quad
  \EE{H_{12}^2} \geq \frac{1}{4 \EE{w_1^2 w_2^2}}
  \geq \frac{1}{4 \min \{\lambda_1, \lambda_2\}}
\]
where the inequality on the left is Cauchy--Schwarz, while the inequality on the right 
uses 
\[
  \EE{w_1^2 w_2^2} \leq \EE{w_1^2} = \lambda_1, \qquad
  \EE{w_1^2 w_2^2} \leq \EE{w_2^2} = \lambda_2.
\]
From the assumption $\L \succeq \alpha \I$ we have $\ell \geq \alpha$,
which gives
\[
  \EE{\tr(\W \tL^2)} = \EE{\tr(\W \ell^2 \H^2)} \geq \alpha^2 \EE{\tr(\W \H^2)}.
\]
Since
\begin{align*}
  \W \H^2 &=
  \begin{bmatrix} \lambda_1 & 0 \\ 0 & \lambda_2 \end{bmatrix}
  \begin{bmatrix} H_{11} & H_{12} \\ H_{12} & H_{22} \end{bmatrix}
  \begin{bmatrix} H_{11} & H_{12} \\ H_{12} & H_{22} \end{bmatrix} \\
  &=
  \begin{bmatrix} \lambda_1 (H^2_{11} + H_{12}^2) & \lambda_1(H_{11}H_{12} + H_{12}H_{22}) \\ 
  \lambda_2(H_{11} H_{12} + H_{12} H_{22}) & \lambda_2(H^2_{22} + H_{12}^2) \end{bmatrix}, 
\end{align*}
this implies
\[
\EE{\tr(\W \H^2)}
 = \lambda_1 \EE{H_{11}^2} + \lambda_2 \EE{H_{22}^2}
 + \EE{H_{12}^2} \geq \EE{H_{12}^2} \geq \frac{1}{4 \min \{\lambda_1, \lambda_2\}},
\]
and therefore $\EE{\tr(\W \tL^2)} \geq \frac{\gamma^2}{4 \lambda_{\min}(\W)}$.

\paragraph{Tsallis regularizers.} A similar analysis can be done for the case of matrix Tsallis regularizers
$R(\U) = -\tr(\U^{\alpha})$ with $\alpha \in (0,1)$, which are related to Tsallis entropy
\citep{NIPS2015_6030,ALO2015,ALSW2017-experiment2}. In this case the variance term $\tr(\W_t \tL_t^2)$
in Matrix Hedge can be replaced by 
the squared \emph{local norms} of the losses \citep{ShaiBook,banditbook,Hazan_OCO}, defined
as $\nabla^{-2} R(\W_t) [\tL_t, \tL_t]$.
where $\nabla^{-2} R$ is the inverse Hessian of the regularizer. As the Tsallis regularizer
is a symmetric spectral function, one can get a closed-form expression for the quadratic form of its Hessian \citet{spectralfunctions}.
Employing convex duality (by identifying $\nabla^{-2} R$ with $\nabla^2 R^*$, where $R^*$ is the convex conjugate of $R$),
and lower bounding, one arrives at the following simple bound on the local norm: 
\[
\nabla^{-2} R(\W_t) [\tL_t, \tL_t] \geq c \tr(\W_t \tL_t \W_t^{1-\alpha} \tL_t).
\]
It is known that the negative entropy is the $\alpha \to 1$ limit of (properly normalized) Tsallis regularizer,
while the $\alpha \to 0$ limit is the log-determinant regularizer. Interestingly, the expression 
above indeed turns into the MH variance term $\tr(\W_t \tL_t^2)$ for $\alpha = 1$, and to the term 
$\tr(\B_t^2)$ with $\B_t = \W_t^{1/2} \tL_t \W_t^{1/2}$ for $\alpha = 0$, which
we encountered in our proofs (compare with \eqref{eq:Bdecomp} for $\eta \to 0$).

One can repeat the same arguments as in the case of the MH variance term to obtain the lower bound
\[
\EEs{\tr(\W_t \tL_t \W_t^{1-\alpha} \tL_t)}{t} \geq \frac{c}{\left(\lambda_{\min}(\W_t) \right)^{\alpha}},
\]
as long as the loss estimate is unbiased and has the same general form as in the MH case. 
This suggests that the only way to control these local norms is to take $\alpha = 0$, resulting in 
the log-determinant regularizer that we use in our main algorithms in the present paper.

We would like to stress one more time that the above arguments do not constitute a lower bound on 
the performance of these algorithms; we merely lower-bound the terms from which all known upper 
bounds are derived for linear bandit problems. At best, this suggests that significantly new 
techniques are required to prove positive results about these algorithms. We ourselves are, 
however, more pessimistic and believe that these algorithms cannot provide regret guarantees of 
optimal order.

\end{document}